\documentclass{article}
\PassOptionsToPackage{numbers,sort&compress}{natbib}
\usepackage[main,preprint]{neurips_2026}
\usepackage[utf8]{inputenc}
\usepackage[T1]{fontenc}

\usepackage{url}
\usepackage{multirow}
\usepackage{booktabs}
\usepackage{microtype}
\usepackage{xcolor}
\usepackage{graphicx}
\usepackage{multirow}
\usepackage{wrapfig}
\usepackage{placeins}

\usepackage{amsmath,amssymb,amsfonts,amsthm,bm}
\usepackage{makecell}
\usepackage{mathtools}
\usepackage{nicefrac}
\usepackage{enumitem}

\usepackage{booktabs}
\usepackage{tabularx}
\usepackage{array}
\usepackage{ragged2e}

\usepackage{algorithm}
\usepackage{algpseudocode}
\usepackage{amsmath}

\usepackage{colortbl}
\usepackage{overpic}

\definecolor{best}{HTML}{FFC7CE}   %
\definecolor{second}{HTML}{FFDCB2} %
\definecolor{third}{HTML}{FFFACD}  %

\usepackage{array}
\usepackage{longtable}
\usepackage{adjustbox}

\newcommand{\best}[1]{\textbf{#1}}

\newcommand{\na}{--}

\makeatletter
\newcommand{\CenterState}[1]{%
  \Statex
  \hskip-\ALG@thistlm
  \parbox{\dimexpr\linewidth+\ALG@thistlm\relax}{\centering #1\par}%
}
\makeatother

\newcolumntype{Y}{>{\RaggedRight\arraybackslash}X}

\algrenewcommand\algorithmicrequire{\textbf{Input:}}
\algrenewcommand\algorithmicensure{\textbf{Output:}}

\newtheorem{proposition}{Proposition}
\newtheorem{remark}{Remark}
\newtheorem{definition}{Definition}

\newcommand{\R}{\mathbb{R}}

\usepackage{hyperref}
\hypersetup{
  hidelinks,
  pdftitle={Birth-Death Structural Learning for 3D Gaussian Splatting},
  pdfauthor={Quan Tran Hong, Long Nguyen-Chi, Binh T. Nguyen}
}

\title{Birth-Death Structural Learning for 3D Gaussian Splatting\thanks{Code is available at: \url{https://github.com/quanth238/birth-death-3dgs}.}}

\author{%
  \normalfont
  \begin{tabular}[t]{@{}c@{\hspace{1em}}c@{\hspace{1em}}c@{}}
    \textbf{Quan Tran Hong} & \textbf{Long Nguyen-Chi} & \textbf{Binh T. Nguyen} \\
    \texttt{26quan.th@vinuni.edu.vn} &
    \texttt{lc.stilldomath@gmail.com} &
    \texttt{binh.nt2@vinuni.edu.vn} \\
    \multicolumn{3}{c}{Center for AI Research, VinUniversity} \\
    \multicolumn{3}{c}{Hanoi, Vietnam}
  \end{tabular}
}

\begin{document}
\maketitle

\begin{abstract}
3D Gaussian Splatting (3DGS) has become a standard representation for real-time novel-view synthesis. Yet, its quality--efficiency trade-off still relies heavily on adaptive density control. Existing cloning, splitting, and pruning rules use proxy signals, such as position gradients, opacity, and image-space error. Although effective in practice, these signals do not directly answer the structural question behind density control: where should new Gaussians be allocated to reduce the multi-view reconstruction loss, and which existing primitives can be pruned with minimal penalty. To address this, we propose first-variation birth–death control for 3DGS, a principled approach that replaces heuristic structural decisions with a variational score derived from the reconstruction loss. At each structural update, we freeze the current Gaussian population, attach auxiliary mass coordinates to each primitive, and differentiate the multiview loss with respect to these coordinates. We demonstrate that this score is a coordinate of the first variation of the Splat Regression Model, giving a rigorous local descent surrogate for birth-death moves: high-score Gaussians are selected as birth sites, while low-score, low-opacity Gaussians are selected for removal. Across standard 3DGS benchmarks, integrating this score into 3DGS-MCMC yields state-of-the-art reconstruction quality and outperforms existing methods, delivering the most significant gains under tight Gaussian budgets. Ultimately, our work provides compelling evidence that gradient-based variational metrics offer a principled alternative to proxy heuristics in 3DGS.
\end{abstract}

\section{Introduction}
\label{sec:introduction}

High-quality photorealistic scene modeling from images is a fundamental task in computer vision and graphics, with applications spanning immersive rendering, robotics, digital humans, and geometric reconstruction~\cite{ChenWang2024Survey3DGS, Zhu2024RoboticsSurvey3DGS}. 3D Gaussian Splatting (3DGS)~\cite{Kerbl2023GaussianSplatting} has emerged as a fast explicit radiance-field representation that models scenes as anisotropic Gaussian primitives and renders them by differentiable rasterization. Its speed and explicitness have led to rapid extensions in dynamic rendering, anti-aliasing, generative 3D content, and SLAM~\cite{Yan2024StreetGaussians, Yu2024MipSplatting, Tang2024DreamGaussian, Qian2024GaussianAvatars, Zhu2024EndoGS, Shao2024SplattingAvatar, Keetha2024SplaTAM, Franke2025VRSplatting, Guedon2024SuGaR}.

Despite this progress, 3DGS training still relies on heuristic \emph{adaptive density control} (ADC). Training a 3DGS model interleaves two distinct optimization mechanisms. Continuous optimization adjusts the parameters (position, covariance, color) of existing Gaussians to fit the image data. Structural control, meanwhile, changes the population itself by creating or destroying primitives. In standard 3DGS, these population updates rely on indirect proxies: high accumulated position gradients trigger cloning or splitting, while low opacity triggers pruning~\cite{Kerbl2023GaussianSplatting}. These signals correlate with regions needing refinement, but they do not directly measure how a discrete birth or death event affects the photometric objective. As a result, primitive placement is sensitive to hand-tuned thresholds. The model frequently misses high-frequency details, creates floaters or needle-like artifacts, and wastes redundant Gaussians to regions that are already well covered~\cite{Bulo2024RevisingDensification, Zhang2024PixelGS, Grubert2025ImprovingADC}. The cost is most visible under tight primitive budgets, where redundant placement directly competes with missing detail~\cite{Baranowski2025ConeGS, Kotovenko2025EDGS, Kheradmand2024MCMC}.

While recent methods improve these signals through pixel-aware gradients, error-guided insertion, and budgeted densification~\cite{Zhang2024PixelGS, Baranowski2025ConeGS, Mallick2024Taming3DGS}, they remain mathematically decoupled from the empirical risk. The Splat Regression Model (SRM) framework~\cite{DanielsRigollet2026SRM} provides a rigorous lens to understand this limitation: it models splat fitting as optimization over mixing measures structured by Wasserstein-Fisher-Rao (WFR) geometry. Under this view, standard 3DGS successfully performs continuous \emph{transport} (moving Gaussian parameters via gradient descent), but completely ignores the WFR \emph{mass} component when deciding where to spawn or destroy primitives. Because disconnected heuristics govern these mass perturbations, a fundamental allocation problem remains unresolved: \emph{How can we identify which discrete primitives to kill, and where to spawn new ones guided by a first-order descent surrogate for the multiview reconstruction loss?}

\paragraph{Contributions.} In this work, we propose reframing 3DGS structural updates from a collection of proxy heuristics into a formal component of empirical risk minimization. Motivated by the WFR geometry, we formulate structural control as a finite-particle birth--death process. At each control step, we freeze the continuous scene parameters and attach an auxiliary mass coordinate to each Gaussian. Differentiating the multiview reconstruction risk with respect to these coordinates yields a per-primitive first-variation score. This score is the negative directional derivative of the objective, evaluated at the frozen parameter state. It mathematically predicts the first-order descent achieved by applying an infinitesimal mass perturbation to that specific primitive. We then use score gaps to guide the birth-death process: capacity is reclaimed by killing low-score, low-opacity Gaussians and reborn at high-score sites. This provides a criterion for birth--death decisions. Because finite structural updates lack strict continuous descent guarantees, we execute them via the state transitions of 3DGS-MCMC~\cite{Kheradmand2024MCMC}. This work bridges the theoretical view of splat fitting as mixing-measure optimization with practical 3DGS training. Our contributions are:
\begin{itemize}[leftmargin=2em] 
    \item We replace heuristic proxy signals with a first-variation score for 3DGS structural control, directly tying primitive allocation to the multiview empirical risk. 
    \item We theoretically justify this framework by discretizing the SRM first variation, and by proving that restricting death to low-opacity Gaussians bounds the immediate rendering disturbance.
    \item We demonstrate that variational allocation outperforms proxy heuristics when representation capacity is constrained. Across standard benchmark datasets, our method achieves state-of-the-art reconstruction quality against matched 3DGS-MCMC and recent budget-aware baselines. 
\end{itemize}

\subsection{Related Work}
\label{sec:related_work}

\paragraph{Adaptive density control and densification.}
The original 3DGS framework couples continuous photometric optimization with adaptive density control: accumulated position-gradient statistics trigger cloning or splitting, while low opacity triggers pruning~\cite{Kerbl2023GaussianSplatting}. A massive body of recent work improves this controller. \emph{Revising Densification} and Pixel-GS replace or reweight the growth criterion using pixel-level error or pixel-aware gradients~\cite{Bulo2024RevisingDensification,Zhang2024PixelGS}, while AbsGS and Improved ADC refine gradient aggregation, threshold schedules, and pruning rules~\cite{Ye2024AbsGS,Grubert2025ImprovingADC}. Other methods manage the primitive budget explicitly through resource-aware schedules, progressive propagation, or high-error insertion rules~\cite{Mallick2024Taming3DGS,Cheng2024GaussianPro,Baranowski2025ConeGS}. While these methods engineer better proxy signals or insertion heuristics, our work diverges fundamentally: we abandon proxies in favor of the structural derivative. By explicitly differentiating the multiview risk with respect to Gaussian mass, we formalize capacity allocation as a direct objective rather than a heuristic byproduct.

\paragraph{Relocation, MCMC, and compact representations.}
Several methods improve 3DGS efficiency by pruning redundant Gaussians or compressing attributes post-training~\cite{Fan2024LightGaussian,Hanson2025PUP3DGS,Niedermayr2024Compressed3DGS,Lee2024Compact3DGaussianRF}, while others avoid incremental growth entirely by relying on dense initializations~\cite{Kotovenko2025EDGS}. The closest backend to our method is 3DGS-MCMC~\cite{Kheradmand2024MCMC}, which rewrites densification and pruning as relocation/addition transitions. Our approach adopts these discrete transitions from 3DGS-MCMC, but completely replaces its structural selection statistic. Rather than relying on opacity-sampling to blindly select relocation targets, we use our first-variation score to predict where released capacity is most likely to reduce the reconstruction risk under a local first-order model.

\paragraph{Variational splat models and first variation.} 
The Splat Regression Model (SRM) formalizes splat models as optimization over mixing measures, deriving Wasserstein--Fisher--Rao gradient flows and explicit first-variation characterizations of mass rescaling~\cite{DanielsRigollet2026SRM}; related birth--death discretizations also appear in sampling theory~\cite{LuSlepcevWang2023BirthDeath}. However, while SRM establishes the continuous, measure-level WFR geometry, it does not prescribe a finite, budget-constrained training algorithm for discrete rasterizers. Our contribution bridges this gap: we derive a concrete, renderer-level mass-coordinate score and a finite score-gap birth--death controller, making this continuous variational geometry actionable for practical 3DGS training.

\section{Background}
\label{sec:background}

We fix notation for the 3DGS renderer and the 3DGS-MCMC relocation backend. The SRM first-variation viewpoint is recalled at the end of the section.

\paragraph{3D Gaussian Splatting.}
At iteration $t$, a scene is represented by a finite population
$\mathcal G_t=\{g_i\}_{i=1}^{M_t}$, where
\[
    g_i=(\mu_i,s_i,R_i,\alpha_i,c_i),
    \qquad
    \mu_i\in\mathbb R^3,\quad
    s_i\in\mathbb R_{>0}^3,\quad
    R_i\in SO(3),\quad
    \alpha_i\in(0,1).
\]

Here $s_i$ and $R_i$ determine the anisotropic covariance $\Sigma_i=R_i\operatorname{Diag}(s_i^2)R_i^\top$, and $c_i$ denotes the view-dependent color parameters.  For a fixed training view $n$ and pixel $x$, let $G_i(x)$ be the projected Gaussian footprint, $a_i(x):=\alpha_iG_i(x)$, and let $j\prec_n i$ denote front-to-back ordering along the ray. Suppressing the view index on projected footprints and transmittances, the rendered color is
\begin{equation}
\hat I_n(x) = \sum_{i=1}^{M_t} T_i(x)\,a_i(x)\,c_i(\mathbf d_n), \qquad T_i(x) := \prod_{j\prec_n i}\bigl(1-a_j(x)\bigr).
\label{eq:bg_alpha_compositing_compact}
\end{equation}

We write $\hat I_n=\Phi(\mathcal G_t,n)$ for the rasterized image.  At
training iteration $t$, a view $n_t$ is sampled and the photometric loss is
$\widehat{\mathcal L}_{\mathrm{img}}^{(t)}
=\ell_{\mathrm{1}}(\Phi(\mathcal G_t, n_t),I_{n_t})$.
The 3DGS-MCMC backend additionally uses opacity and scale regularizers to encourage unused Gaussians to disappear; we retain these terms in training but do not expand them here, since our score pass is defined from the image objective in Section~\ref{sec:method}.

\paragraph{3DGS-MCMC: Markov view and relocation/addition operations.}
3DGS-MCMC~\citep{Kheradmand2024MCMC} interprets the Gaussian population as a sample state of a Markov chain.  For a regularized training objective $\mathcal L$, the formal target density is
\begin{equation*}
    p_\beta(\mathcal G)
    =
    Z_\beta^{-1}\exp\!\bigl(-\beta\mathcal L(\mathcal G)\bigr),
\label{eq:bg_mcmc_target}
\end{equation*}
where $Z_\beta$ is the normalizing constant and $\beta>0$ is an inverse
temperature.  Under this view, the usual gradient update is interpreted as the zero-noise limit of an SGLD step.  For a parameter block $\theta$,
\begin{equation*}
    \theta \leftarrow \theta - \eta\,\widehat{\nabla}_\theta\mathcal L(\mathcal G) + \sigma_\theta(\mathcal G)\,\xi, \qquad \xi\sim\mathcal  N(0,I),
\label{eq:bg_sgld}
\end{equation*}
where $\widehat{\nabla}$ is a stochastic gradient and
$\sigma_\theta(\mathcal G)$ is the exploration scale.  In the original
3DGS-MCMC implementation, this noise is applied to Gaussian centers, with scale depending on the current Gaussian state.

The structural mechanism of 3DGS-MCMC realizes densification and pruning through deterministic state transitions. Let $\mathcal D_t:=\{i:\alpha_i\le\alpha_{\mathrm{death}}\}$ 
and $\mathcal A_t:=\{i:\alpha_i>\alpha_{\mathrm{death}}\}$ denote the candidate dead and live pools, 
respectively. To reallocate capacity, $r_j$ recycled slots from $\mathcal D_t$ (or fresh slots 
during addition) are assigned to a target Gaussian $j$, yielding $N_j:=r_j+1$ co-located copies 
parameterized as:
\begin{equation}
    \mu_{j,\ell}'=\mu_j,
    \qquad
    \alpha_{j,\ell}'
    =
    1-\bigl(1-\alpha_j\bigr)^{1/N_j},
    \qquad
    \Sigma_{j,\ell}'
    =
    \gamma(N_j,\alpha_j)\,\Sigma_j ,
\label{eq:bg_mcmc_relocation}
\end{equation}
for $\ell=1,\dots,N_j$, with color and orientation inherited from the target. The closed-form 
scalar $\gamma(N_j,\alpha_j)$ rescales the covariance such that the aggregate opacity footprint 
of the co-located cluster matches that of the original primitive, minimizing instantaneous rendering 
perturbations.

In our framework, we decouple structural decision-making from physical execution: we adopt these 
closed-form transitions strictly as a \emph{conservative realization backend}, while entirely replacing 
the heuristic opacity proposals with our objective-derived first-variation score.

\paragraph{SRM viewpoint: splat measures and first variation.}
Splat Regression Models (SRM)~\citep{DanielsRigollet2026SRM} formalize ``splat'' representations as optimization over mixing measures.
Let $\rho$ be a centered isotropic ``mother'' density on $\R^3$.
For $(A,b)\in\R^{3\times 3}\times\R^3$, define the transformed splat $\rho_{A,b}:=(A(\cdot)+b)_{\#}\rho$.
A (vector-valued) splat measure $\mu$ over weights and splat parameters induces the mixture model
\begin{equation*}
f_\mu(x)
:=
\int v\,\rho_{A,b}(x)\,\mu(dv,dA,db),
\qquad x\in\R^3.
\label{eq:bg_srm_model}
\end{equation*}
SRM endows the space of such measures with a Wasserstein--Fisher--Rao (WFR) geometry in which gradient flow decomposes into a \emph{transport} component (moving $A,b$) and a \emph{mass} component (rescaling $v$). The first-order object governing the mass component is the \emph{first variation} of the data-fit functional. Inspired by the SRM first-variation view of splat measures, our score pass instantiates a finite-dimensional coordinate of this first-variation field for the specific mass direction that scales opacities (see Section~\ref{subsec:score_pass}).

\begin{figure}[t]
    \centering
    \includegraphics[width=\linewidth]{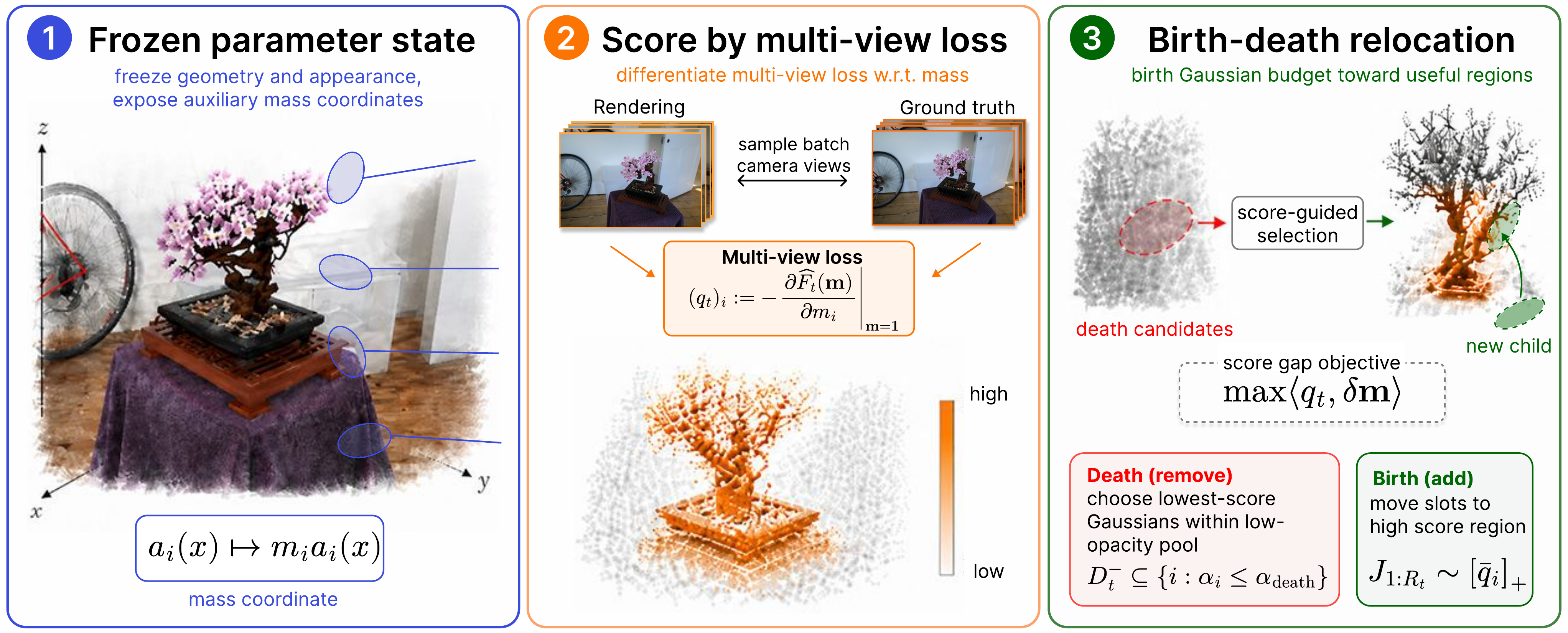}
    \vspace{-1.0em}
    \caption{Overview of the score-guided birth--death controller. At each structural update, we freeze the current Gaussian parameters and expose one auxiliary mass coordinate per primitive, implemented as \(a_i(x)\mapsto m_i a_i(x)\). A batched multiview backward pass computes the first-variation score \(q_i=-\partial \widehat F_t(\mathbf m)/\partial m_i|_{\mathbf m=\mathbf 1}\). Higher scores explicitly predict a larger first-order loss reduction from assigning mass to that primitive. Guided by these scores, the controller kills low-opacity, low-score slots and births new capacity at high-score parent sites using the 3DGS-MCMC transitions.}
    \label{fig:pipeline}
\end{figure}

\section{Method}
\label{sec:method}
\subsection{First-Variation Score from Splat Measures}
\label{subsec:score_pass}

Let $F_{\mathcal D}(\mathcal G) := N^{-1}\sum_{n=1}^N \ell_{\mathrm{img}}(\Phi(\mathcal G, n),I_n)$ denote the full multiview empirical risk. The Splat Regression Model (SRM) shows that optimizing a splat-based representation involves two distinct operations: moving parameters (transport) and changing primitive weights (mass). Standard 3DGS handles transport natively via continuous gradient descent. To handle mass, we evaluate how infinitesimal changes to a Gaussian's weight affect the empirical risk.

At a controller step $t$, we freeze the current population $\mathcal G_t=\{g_i\}_{i=1}^{M_t}$ and attach an auxiliary mass coordinates $\mathbf m_i$ for each Gaussian, initialized to the all-ones vector $\mathbf{1}$. We define a mass-perturbed population $\mathcal G_t^{\mathbf m}$ by scaling the opacity contribution of each Gaussian $i$ by $m_i$, i.e., $a_i(x)\mapsto m_i a_i(x)$. We estimate the minibatch multiview risk and its negative mass derivative from a batch of views $\mathcal B_t$:
\begin{equation}
\widehat F_t(\mathbf m)
:=
\frac{1}{B}
\sum_{n\in\mathcal B_t}
\ell_{\mathrm{img}}\!\left(
\Phi(\mathcal G_t^{\mathbf m}, n), I_n
\right),
\qquad
(q_t)_i
:=
-
\left.
\frac{\partial \widehat F_t(\mathbf m)}
{\partial m_i}
\right|_{\mathbf m=\mathbf 1}.
\label{eq:method_first_variation_score}
\end{equation}
All scene parameters are detached during this computation, so a single backward pass returns the score vector $q_t$. The sign of $(q_t)_i$ gives the local mass interpretation: increasing $m_i$ decreases the empirical risk by $\varepsilon (q_t)_i + o(\varepsilon)$. 

Fix a view $(I_n,\pi_n)$. Under mass coordinates $\mathbf m$, the renderer induces a map $f_{n,t}(\cdot;\mathbf m) := \Phi(\mathcal G_t^{\mathbf m},\pi_n) \in L^2(\Omega;\mathbb R^3)$. Define the per-view functional $F_n(f):=\ell_{\mathrm{img}}(f,I_n)$. Following SRM, the relevant first-order object governing mass creation is the first variation.

\begin{definition}[First variation on images {\cite[Def.~2]{DanielsRigollet2026SRM}}]
\label{def:first_variation}
Let $F:L^2(\Omega;\mathbb R^3)\to\mathbb R$ be Fr\'echet differentiable at
$f$.  Its first variation $\delta F[f]\in L^2(\Omega;\mathbb R^3)$ is the
Riesz representative of the differential: for all
$\xi\in L^2(\Omega;\mathbb R^3)$,
\[
\left.\frac{d}{d\varepsilon}F(f+\varepsilon\xi)\right|_{\varepsilon=0}
=
\int_\Omega \langle \xi(x),\,\delta F[f](x)\rangle\,dx .
\]
\end{definition}

\begin{remark}[On differentiability and discretization]
Definition~\ref{def:first_variation} assumes Fr\'echet differentiability. The practical image loss contains nonsmooth components such as $\ell_1$ and differentiable image-structure terms such as DSSIM.  The implemented renderer is evaluated on a pixel grid, in which case the integrals below are finite sums over pixels.
\end{remark}

Applying the chain rule to $F_n\circ f_{n,t}$ yields a coordinate expression
for the score.

\begin{proposition}[Score as a first-variation coordinate]
\label{prop:qi_first_variation}
Assume each $F_n$ admits a first variation and
$f_{n,t}(\cdot;\mathbf m)$ is differentiable in $\mathbf m$ at
$\mathbf m=\mathbf 1$.  Then the score in
\eqref{eq:method_first_variation_score} satisfies
\[
(q_t)_i
=
-\frac{1}{B}\sum_{n\in\mathcal B_t}
\int_\Omega
\Bigl\langle
\delta F_n\!\bigl[f_{n,t}(\cdot;\mathbf 1)\bigr](x),\;
\left.
\frac{\partial f_{n,t}(x;\mathbf m)}{\partial m_i}
\right|_{\mathbf m=\mathbf 1}
\Bigr\rangle\,dx .
\]
\end{proposition}

The proof is provided in Appendix~\ref{app:deferred_proofs}. Proposition~\ref{prop:qi_first_variation} connects the SRM first-variation viewpoint to renderer-level mass coordinates. Intuitively, the integral pairs the image-space loss gradient with the image change induced by rescaling the $i$-th Gaussian. Ranking Gaussians by $(q_t)_i$ therefore orders the predicted loss decrease for equal infinitesimal mass increments at the frozen state. Unlike accumulated position gradients or plain opacity~\cite{Kerbl2023GaussianSplatting, Kheradmand2024MCMC}, the score directly evaluates this mass direction, providing an objective-derived criterion for allocating capacity.

In the specific context of 3DGS alpha-compositing \citep{Zwicker2002EWA}, applying the chain rule to the renderer gives a concrete pixel-level expression for this score:
\begin{equation}
(q_t)_i
=
-\frac{1}{B}
\sum_{n\in\mathcal B_t}\sum_x
\left\langle
r_n(x),\;
T_i(x)a_i(x)
\bigl(c_i(\mathbf d_n)-\hat I_i^{\mathrm{back}}(x)\bigr)
\right\rangle ,
\label{eq:score_pixel_interpretation}
\end{equation}
where $\hat I_i^{\mathrm{back}}(x)$ is the accumulated background color and $r_n(x)$ is the image-space loss gradient. Equation~\eqref{eq:score_pixel_interpretation} gives the geometric intuition: a Gaussian receives a high score when its footprint is visible, meaning $T_i(x)a_i(x)$ is large; when its color differs from the accumulated background behind it; and when the resulting color change aligns with the negative loss gradient $-r_n(x)$.

\subsection{Birth--Death Allocation via Score Gaps}
\label{subsec:birth_death_score_gap}

The score \(q_t\) provides a continuous mass derivative, but structural control requires a discrete allocation decision. We model reclaiming capacity from death coordinates \(D\) and assigning it to birth coordinates \(J_{1:R_t}\) through the following balanced mass perturbation:
\begin{equation*}
\delta\mathbf m(J,D)
=
\kappa
\left(
\sum_{r=1}^{R_t}e_{J_r}
-
\sum_{i\in D}e_i
\right),
\qquad |D|=R_t .
\label{eq:delta_m}
\end{equation*}
Here $e_i$ is the $i$-th coordinate vector and $\kappa>0$ sets the scale of the mass perturbation. The following bound identifies its first-order loss term under a smooth local model.

\begin{proposition}[Local descent upper bound under smoothness]
\label{prop:descent_lemma}
If $F(\mathbf m):=\widehat F_t(\mathbf m)$ is $L_{\mathrm{sm}}$-smooth on a neighbourhood containing the segment $\{\mathbf 1+\tau\delta\mathbf m:\tau\in[0,1]\}$, then
\[
F(\mathbf 1+\delta\mathbf m)
\le
F(\mathbf 1)
-
\langle q_t,\delta\mathbf m\rangle
+
\frac{L_{\mathrm{sm}}}{2}\|\delta\mathbf m\|_2^2 .
\]
\end{proposition}

Proposition~\ref{prop:descent_lemma} (proof in Appendix~\ref{app:deferred_proofs}) identifies the score gap as the first-order predicted loss decrease:
\begin{equation}
\langle q_t,\delta\mathbf m\rangle
=
\kappa
\left(
\sum_{r=1}^{R_t}(q_t)_{J_r}
-
\sum_{i\in D}(q_t)_i
\right).
\label{eq:score_gap_inner_product}
\end{equation}

\begin{proposition}[Bounded influence of low-opacity Gaussians]
\label{prop:low_opacity_bound}
Fix a view $n$ and let $\hat I_n$ be the rendered image from
\eqref{eq:bg_alpha_compositing_compact}.  Let $\hat I_n^{(-i)}$ be the image
obtained by setting $\alpha_i=0$ while keeping all other Gaussians unchanged.
Assume rendered colors and the terminal background color are bounded in
$[0,1]^3$, and define the per-image mean absolute error
\[
\|J\|_{1,\mathrm{mean}}
:=
\frac{1}{|\Omega_{\mathrm{pix}}|}
\sum_{x\in\Omega_{\mathrm{pix}}}
\frac{1}{3}\|J(x)\|_1 .
\]
Then
\[
\|\hat I_n-\hat I_n^{(-i)}\|_{1,\mathrm{mean}}
\le
\alpha_i .
\]
Consequently, if $\ell_{\mathrm{img}}(\cdot,I_n)$ is $L_\ell$-Lipschitz with
respect to $\|\cdot\|_{1,\mathrm{mean}}$ in a neighbourhood of $\hat I_n$,
then
\[
\bigl|
\ell_{\mathrm{img}}(\hat I_n,I_n)
-
\ell_{\mathrm{img}}(\hat I_n^{(-i)},I_n)
\bigr|
\le
L_\ell\,\alpha_i .
\]
\end{proposition}

Proposition~\ref{prop:low_opacity_bound} (proof in Appendix~\ref{app:deferred_proofs}) bounds the disturbance from an individual low-opacity removal. This motivates the candidate pool $\mathcal P_t:=\{i:\alpha_i\le\alpha_{\mathrm{death}}\}$. Within this pool, smaller scores indicate lower retention value under the local mass model.

The final variational controller merges this continuous ideal with discrete constraints. We stabilize multi-view stochasticity using an exponentially moving averaged (EMA) score \(\bar q_t\). The death set \(\mathcal D_t^-\) selects the \(R_t\) Gaussians with the lowest scores \((q_t)_i\) \emph{within} the pool \(\mathcal P_t\). The birth multiset \(J_{1:R_t}\) is sampled from live Gaussians proportional to their positive scores \([(\bar q_t)_i]_+\).

\subsection{Discrete Realization by the 3DGS-MCMC Kernels}
\label{subsec:mcmc_realization}
\begin{wrapfigure}[14]{r}{0.54\textwidth}
\vspace{-1.5em}
\footnotesize
\hrule
\vspace{0.5em}
\refstepcounter{algorithm}
\label{alg:first_variation_birth_death}
\noindent\textbf{Algorithm~\thealgorithm\quad First-variation controller.}
\vspace{0.3em}
\hrule
\vspace{0.35em}

\begin{algorithmic}[1]
\Require $\Delta$, $B$, $\alpha_{\mathrm{death}}$
\For{$t=1,2,\ldots$}
    \State Take a gradient step on continuous parameters of $\mathcal G_t$
    \If{$t \bmod \Delta = 0$}
        \State Freeze $\mathcal G_t$ and sample multi-view batch $\mathcal B_t$.
        \State Compute $q_t$ \Comment{Eq.~\eqref{eq:method_first_variation_score}}
        \State Form $\mathcal P_t = \{i : \alpha_i \le \alpha_{\mathrm{death}}\}$.
        \State Select $\mathcal D_t^-$ as the lowest-score elements in $\mathcal P_t$
        \Comment{Eq.~\eqref{eq:score_gap_inner_product}}
        \State Sample birth sites $J_{1:R_t}$ from live Gaussians using $[(q_t)_i]_+$
        \Comment{Eq.~\eqref{eq:score_gap_inner_product}}
        \State Apply $\mathcal K_{\mathrm{reloc}}$
        \Comment{Eq.~\eqref{eq:mcmc_realization}}
    \EndIf
\EndFor
\end{algorithmic}

\vspace{0.3em}
\hrule
\vspace{0.5\baselineskip}
\end{wrapfigure}

Our variational policy dictates \emph{which} primitives are retired and \emph{where} new capacity is sampled to approximate first-order loss descent. Once the death set $\mathcal D_t^-$ is selected from the lowest-scoring candidates in $\mathcal P_t$ and the birth multiset $J_{1:R_t}$ is sampled proportional to positive scores $[(\bar{q}_t)_i]_+$, the discrete state transition is realized by the footprint-preserving relocation kernel:
\begin{equation}
\mathcal G_{t+}
=
\mathcal K_{\mathrm{reloc}}
\bigl(
\mathcal G_t; \,
\mathcal D_t^-, \,
J_{1:R_t}
\bigr).
\label{eq:mcmc_realization}
\end{equation}
Each recycled slot in $\mathcal D_t^-$ is transferred to a sampled birth site in $J_{1:R_t}$, where the closed-form opacity and covariance corrections~\eqref{eq:bg_mcmc_relocation} preserve the aggregate optical footprint to minimize instantaneous rendering disturbance. When the population is below the prescribed cap, the addition kernel directly activates new slots at sites sampled from $[(\bar{q}_t)_i]_+$ without requiring paired deaths.

\section{Experiments}
\label{sec:experiments}

Our experiments test a single claim: the first variation of the multiview reconstruction risk is a better birth--death signal than opacity or gradient proxies.  We fix the renderer, optimizer, and MCMC transitions to isolate this statistic. This allows us to answer three direct questions: Does the variational controller improve reconstruction quality at the same primitive budgets? Is the gain explicitly caused by the WFR score? And does the allocation behavior align with local descent predictions?

\subsection{Experimental setup}
\label{subsec:experimental_settings}

\paragraph{Datasets and metrics.}
We evaluate on Mip-NeRF360~\cite{barron2022mipnerf360}, Tanks \& Temples~\cite{knapitsch2017tanks}, and DeepBlending~\cite{hedman2018deepblending}. Our main fixed-budget Mip-NeRF360 averages use seven scenes: Bicycle, Garden, Stump, Room, Counter, Kitchen, and Bonsai; Flowers and Treehill are excluded. Tanks \& Temples uses Truck and Train, and DeepBlending uses Dr.\ Johnson and Playroom. We report PSNR, SSIM~\cite{wang2004ssim}, and LPIPS~\cite{zhang2018lpips} on held-out test views. Per-scene results are in Appendix~\ref{app:per_scene_results}; compact-model comparisons use all nine Mip-NeRF360 scenes.

\paragraph{Budgeted evaluation.}
We evaluate under matched Gaussian caps $K\in\{100\mathrm{k},500\mathrm{k},1\mathrm{M}\}$. The fixed-budget setting isolates the structural decision. Every birth consumes finite capacity. Poor allocation cannot be hidden by unbounded growth.

\paragraph{Baselines and protocol.}
We compare against 3DGS~\cite{Kerbl2023GaussianSplatting} with SfM initialization, Improved ADC~\cite{Grubert2025ImprovingADC}, 3DGS-MCMC~\cite{Kheradmand2024MCMC} with SfM initialization, EDGS~\cite{Kotovenko2025EDGS}, and ConeGS~\cite{Baranowski2025ConeGS}. Rows marked $\dagger$ use published fixed-budget benchmark numbers from ConeGS under the same dataset split and Gaussian cap. Claims about the effect of the proposed controller are based on matched 3DGS-MCMC reruns: the renderer, optimizer, initialization, cap, and transition maps remain identical. Only the structural selection statistic changes.

\begin{figure*}[t]
\centering
\includegraphics[width=\textwidth]{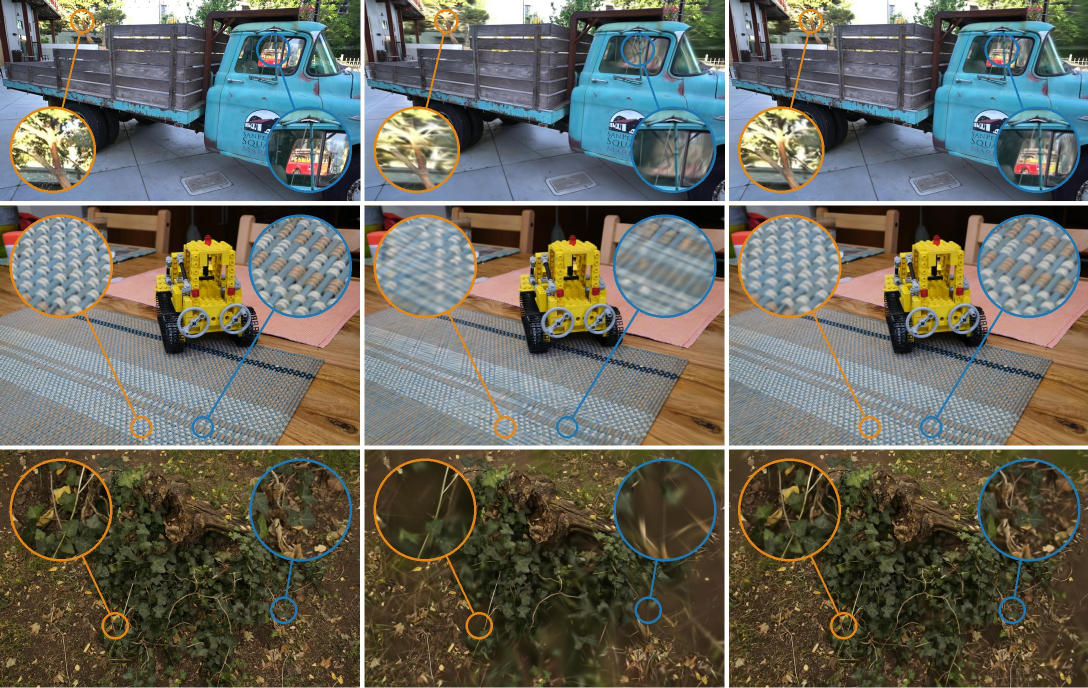}
\par\nobreak\vspace{2pt}
{\small\bfseries
\begin{tabularx}{\linewidth}{@{}*{3}{>{\centering\arraybackslash}X}@{}}
(a) Ground Truth & (b) 3DGS-MCMC & (c) Ours \\
\end{tabularx}\par
}
\caption{Qualitative comparison on Tanks \& Temples (Truck, top row) and Mip-NeRF360 (middle and bottom rows) at a 200k-Gaussian cap with random initialization. All methods use the same view, scale, and crop coordinates in each row. Additional examples are in Appendix~\ref{app:qualitative}.}
\label{fig:main_qualitative}
\end{figure*}

\paragraph{Implementation.}
Our implementation builds on 3DGS-MCMC.  We keep the renderer, training loss, optimizer, relocation/addition kernels, and event schedule unchanged. All methods are trained for $30{,}000$ iterations. Hyperparameters, hardware, and timing details are in Appendix~\ref{app:reproducibility}.

\subsection{Fixed-budget reconstruction quality}
\label{subsec:main_quantitative}

Table~\ref{tab:sota_comparison} establishes fixed-budget reconstruction quality. At the tightest 100k cap, our controller outperforms 3DGS-MCMC baseline, EDGS, and ConeGS across all three datasets. This severe scarcity highlights our advantage: reallocating via WFR score derivatives yields the largest impact.

At 500k and 1M, our method has the highest PSNR across all three datasets in Table~\ref{tab:sota_comparison}. SSIM and LPIPS are metric-dependent: on Mip-NeRF360 at 500k, ConeGS has slightly higher SSIM (0.870 vs.\ 0.869), and competing methods achieve lower LPIPS at both larger caps. These results support the benefit of score-guided allocation under tight budgets without implying uniform superiority. Figure~\ref{fig:main_psnr_budget_frontier} shows the matched PSNR--budget frontier. Full metrics, random-initialization results, and runtimes are in Appendix~\ref{app:budget_runtime}.

\begin{table}[t]
\caption{
Fixed-budget comparison under matched Gaussian caps with SfM initialization for Ours. Mip-NeRF360 averages use the seven scenes listed in Section~\ref{subsec:experimental_settings}. Higher PSNR and SSIM are better; lower LPIPS is better. We highlight the {\setlength\fboxsep{1pt}\colorbox{best}{best}}, {\setlength\fboxsep{1pt}\colorbox{second}{second best}} and {\setlength\fboxsep{1pt}\colorbox{third}{third best}} results within each dataset-budget block.
}
\label{tab:sota_comparison}
\centering
\resizebox{\textwidth}{!}{%
\begin{tabular}{l ccc ccc ccc}
\toprule
\multirow{2}{*}{\textbf{Method}} &
\multicolumn{3}{c}{\textbf{Mip-NeRF360}} &
\multicolumn{3}{c}{\textbf{Tanks \& Temples}} &
\multicolumn{3}{c}{\textbf{DeepBlending}} \\
\cmidrule(lr){2-4} \cmidrule(lr){5-7} \cmidrule(lr){8-10}
& PSNR $\uparrow$ & SSIM $\uparrow$ & LPIPS $\downarrow$
& PSNR $\uparrow$ & SSIM $\uparrow$ & LPIPS $\downarrow$
& PSNR $\uparrow$ & SSIM $\uparrow$ & LPIPS $\downarrow$ \\
\midrule
\multicolumn{10}{c}{\textit{Number of Gaussians limited to 100k}} \\
\midrule
3DGS~\cite{Kerbl2023GaussianSplatting} (SfM)$^\dagger$ & 23.61 & 0.693 & 0.413 & 22.38 & 0.774 & 0.333 & 24.65 & 0.827 & 0.412 \\
Improved ADC~\cite{Grubert2025ImprovingADC} & 26.34 & 0.781 & 0.322 & 22.08 & \cellcolor{third}0.784 & \cellcolor{second}0.280 & 28.25 & \cellcolor{second}0.885 & \cellcolor{second}0.315 \\
3DGS-MCMC~\cite{Kheradmand2024MCMC} (SfM)$^\dagger$ & 27.06 & \cellcolor{third}0.800 & 0.303 & \cellcolor{third}22.50 & 0.780 & 0.332 & \cellcolor{third}28.94 & 0.876 & 0.333 \\
EDGS~\cite{Kotovenko2025EDGS}$^\dagger$ & \cellcolor{third}27.09 & 0.798 & \cellcolor{third}0.296 & 22.32 & 0.777 & 0.324 & 28.43 & 0.872 & 0.337 \\
ConeGS~\cite{Baranowski2025ConeGS}$^\dagger$ & \cellcolor{second}27.74 & \cellcolor{second}0.809 & \cellcolor{second}0.285 & \cellcolor{second}23.12 & \cellcolor{second}0.791 & \cellcolor{third}0.310 & \cellcolor{second}29.44 & \cellcolor{third}0.880 & \cellcolor{third}0.328 \\
\textbf{Ours} & \textbf{\cellcolor{best}27.94} & \textbf{\cellcolor{best}0.818} & \textbf{\cellcolor{best}0.276} & \textbf{\cellcolor{best}23.60} & \textbf{\cellcolor{best}0.831} & \textbf{\cellcolor{best}0.225} & \textbf{\cellcolor{best}29.45} & \textbf{\cellcolor{best}0.892} & \textbf{\cellcolor{best}0.303} \\
\midrule
\multicolumn{10}{c}{\textit{Number of Gaussians limited to 500k}} \\
\midrule
3DGS~\cite{Kerbl2023GaussianSplatting} (SfM)$^\dagger$ & 28.22 & 0.821 & 0.260 & 23.54 & 0.816 & 0.265 & 29.25 & 0.882 & 0.302 \\
Improved ADC~\cite{Grubert2025ImprovingADC} & 28.46 & 0.843 & 0.232 & 23.66 & \cellcolor{second}0.842 & \cellcolor{second}0.192 & 28.95 & \cellcolor{second}0.896 & \textbf{\cellcolor{best}0.260} \\
3DGS-MCMC~\cite{Kheradmand2024MCMC} (SfM)$^\dagger$ & \cellcolor{third}28.82 & 0.861 & 0.214 & \cellcolor{third}23.68 & 0.825 & 0.259 & \cellcolor{third}29.44 & 0.887 & 0.308 \\
EDGS~\cite{Kotovenko2025EDGS}$^\dagger$ & \cellcolor{third}28.82 & \cellcolor{third}0.865 & \cellcolor{second}0.193 & 23.47 & \cellcolor{third}0.831 & \cellcolor{third}0.229 & 29.33 & 0.889 & 0.286 \\
ConeGS~\cite{Baranowski2025ConeGS}$^\dagger$ & \cellcolor{second}29.08 & \textbf{\cellcolor{best}0.870} & \textbf{\cellcolor{best}0.190} & \cellcolor{second}23.69 & 0.829 & \cellcolor{third}0.229 & \cellcolor{second}29.86 & \cellcolor{third}0.891 & \cellcolor{third}0.285 \\
\textbf{Ours} & \textbf{\cellcolor{best}29.11} & \cellcolor{second}0.869 & \cellcolor{third}0.203 & \textbf{\cellcolor{best}24.05} & \textbf{\cellcolor{best}0.857} & \textbf{\cellcolor{best}0.178} & \textbf{\cellcolor{best}30.04} & \textbf{\cellcolor{best}0.906} & \cellcolor{second}0.262 \\
\midrule
\multicolumn{10}{c}{\textit{Number of Gaussians limited to 1M}} \\
\midrule
3DGS~\cite{Kerbl2023GaussianSplatting} (SfM)$^\dagger$ & 28.71 & 0.846 & 0.222 & 23.60 & 0.826 & 0.245 & 28.93 & 0.881 & 0.292 \\
Improved ADC~\cite{Grubert2025ImprovingADC} & 28.94 & 0.858 & 0.205 & \cellcolor{second}24.06 & \cellcolor{second}0.855 & \cellcolor{second}0.168 & 29.18 & \cellcolor{second}0.902 & \textbf{\cellcolor{best}0.248} \\
3DGS-MCMC~\cite{Kheradmand2024MCMC} (SfM)$^\dagger$ & \cellcolor{third}29.23 & 0.875 & \cellcolor{third}0.190 & \cellcolor{third}23.93 & 0.836 & 0.233 & \cellcolor{third}29.67 & 0.887 & 0.288 \\
EDGS~\cite{Kotovenko2025EDGS}$^\dagger$ & 29.18 & \cellcolor{third}0.877 & \textbf{\cellcolor{best}0.168} & 23.66 & \cellcolor{third}0.842 & \cellcolor{third}0.200 & 29.43 & \cellcolor{third}0.889 & \cellcolor{second}0.271 \\
ConeGS~\cite{Baranowski2025ConeGS}$^\dagger$ & \cellcolor{second}29.37 & \cellcolor{second}0.880 & \textbf{\cellcolor{best}0.168} & 23.70 & 0.835 & 0.207 & \cellcolor{second}29.69 & \cellcolor{third}0.889 & \cellcolor{third}0.273 \\
\textbf{Ours} & \textbf{\cellcolor{best}29.49} & \textbf{\cellcolor{best}0.883} & \cellcolor{second}0.178 & \textbf{\cellcolor{best}24.17} & \textbf{\cellcolor{best}0.865} & \textbf{\cellcolor{best}0.157} & \textbf{\cellcolor{best}30.14} & \textbf{\cellcolor{best}0.911} & \textbf{\cellcolor{best}0.248} \\
\bottomrule
\end{tabular}
}
\end{table}

\begin{figure*}[t]
\centering
\includegraphics[width=1.0\textwidth]{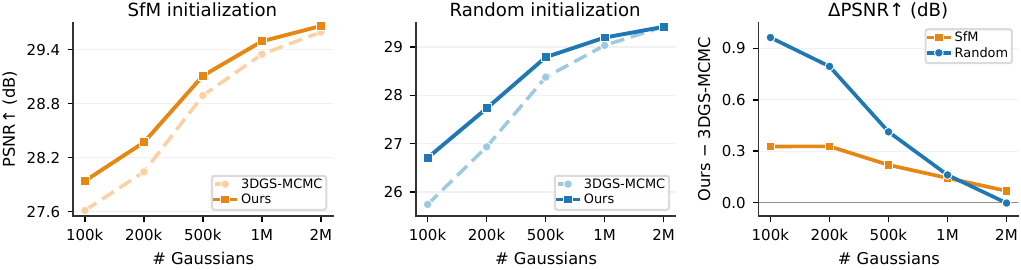}
\vspace{-0.75em}
\caption{
Matched PSNR--budget comparison on the seven-scene Mip-NeRF360 subset. The renderer, optimizer, relocation kernel, and training schedule are fixed; only the birth--death statistic changes from opacity to first variation.
}
\label{fig:main_psnr_budget_frontier}
\end{figure*}

\subsection{Does the first-variation score matter?}
\label{subsec:controller_ablation}

To isolate the effect of the WFR score on these gains, we perform a controlled drop-in ablation on Mip-NeRF360 at a 100k cap. We freeze the renderer, optimizer, training schedule, and MCMC relocation kernel; only the structural selection statistic changes. Table~\ref{tab:controller_ablation} shows that replacing opacity sampling with position gradients yields a +0.04\,dB gain, while the post-activation opacity gradient yields +0.11\,dB. The latter already measures local loss sensitivity, but to an additive opacity change. Our mass score $q_i=-\alpha_i\,\partial\widehat{F}_t/\partial\alpha_i$ instead measures relative mass rescaling and achieves +0.33\,dB, an additional +0.22\,dB over the plain gradient under the matched protocol. Additional score diagnostics, including copy-gain, recycle-hit, and removal-agreement probes, are provided in Appendix~\ref{app:score_validation}.

\begin{table}[!t]
\centering
\begin{minipage}{0.72\linewidth}
\caption{
Controller-statistic ablation on Mip-NeRF360 (100k cap).
Within-sweep gains over the matched opacity-controller rerun.
}
\label{tab:controller_ablation}
\small
\setlength{\tabcolsep}{5pt}
\begin{tabularx}{\linewidth}{@{}l Y r@{}}
\toprule
\textbf{Variant} & \textbf{Selection statistic} & \textbf{$\Delta$PSNR (dB)$\uparrow$} \\
\midrule
(a) MCMC baseline & $\alpha_i$ & -- \\
(b) Position gradient & $\|\nabla_{\mu_i}\widehat{F}_t\|$ & +0.04 \\
(c) Opacity gradient (post-act.) & $-\partial\widehat{F}_t/\partial\alpha_i$ & +0.11 \\
\midrule
(d) Birth only & $(q_i)_+$ & +0.28 \\
(e) Death only & $q_i$ & +0.05 \\
\midrule
\textbf{(f) Ours (Full)} & $q_i$ & \textbf{+0.33} \\
\bottomrule
\end{tabularx}
\end{minipage}
\end{table}

\subsection{Qualitative behavior}
\label{subsec:qualitative}

Figure~\ref{fig:main_qualitative} gives representative crops, while Figure~\ref{fig:main_qualitative_error_maps} adds absolute error maps across scenes and budgets. The matched comparisons show improved reconstruction of fine textures and boundaries with score-guided allocation.

\begin{figure}[!htbp]
\centering
\includegraphics[width=\linewidth]{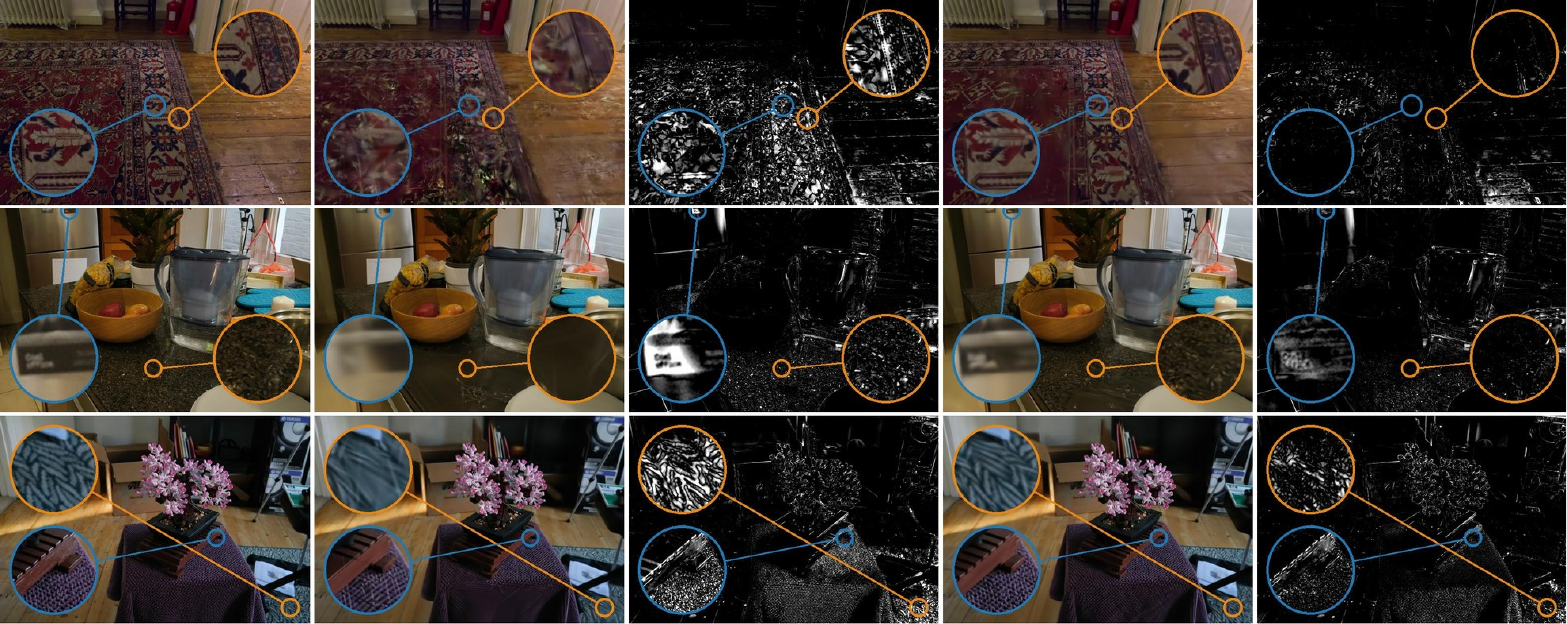}
\par\nobreak\vspace{2pt}
{\small\bfseries
\begin{tabularx}{\linewidth}{@{}*{5}{>{\centering\arraybackslash}X}@{}}
Ground Truth & 3DGS-MCMC & $L_1$ Error & Ours & $L_1$ Error \\
\end{tabularx}\par
}
\caption{Qualitative comparison with absolute $L_1$ error maps. Rows, top to bottom: Dr.\ Johnson (1M), Counter (200k), and Bonsai (100k); all use random initialization. Brighter error-map pixels indicate larger error. Each row uses the same Gaussian cap, test view, and crop across methods.}
\label{fig:main_qualitative_error_maps}
\end{figure}
\FloatBarrier

\providecommand{\rankone}[1]{\cellcolor{best}\textbf{#1}}
\providecommand{\ranktwo}[1]{\cellcolor{second}#1}
\providecommand{\rankthree}[1]{\cellcolor{third}#1}

\subsection{Compaction without Post-Training Pruning}
\label{subsec:compact_regime}
Post-training compaction methods reduce a larger trained model~\cite{Fan2024LightGaussian, Hanson2025PUP3DGS, Xie2024MesonGS}, whereas our method trains under a fixed cap. Table~\ref{tab:compact_regime_main} compares our 500k models with external compact-model results from GHAP~\cite{Wang2025GHAP}. Mip-NeRF360 averages all nine scenes in this comparison, including Flowers and Treehill. Gaussian counts and training protocols differ across rows.

\begin{table*}[!htbp]
\centering

\resizebox{\textwidth}{!}{
\begin{tabular}{l r c c c r c c c r c c c}
\toprule
\multirow{2}{*}{\textbf{Method}} &
\multicolumn{4}{c}{\textbf{Tanks \& Temples}} &
\multicolumn{4}{c}{\textbf{Mip-NeRF360}} &
\multicolumn{4}{c}{\textbf{DeepBlending}} \\
\cmidrule(lr){2-5}
\cmidrule(lr){6-9}
\cmidrule(lr){10-13}
&
\#G & PSNR$\uparrow$ & SSIM$\uparrow$ & LPIPS$\downarrow$ &
\#G & PSNR$\uparrow$ & SSIM$\uparrow$ & LPIPS$\downarrow$ &
\#G & PSNR$\uparrow$ & SSIM$\uparrow$ & LPIPS$\downarrow$ \\
\midrule
Vanilla 3DGS~\citep{Kerbl2023GaussianSplatting}
& 1577 & 23.79 & 0.853 & 0.169
& 2627 & 27.55 & 0.813 & 0.221
& 2475 & 29.82 & 0.907 & 0.238 \\
\midrule
LightGaussian-20\%~\citep{Fan2024LightGaussian}
& 315 & 22.49 & 0.779 & 0.271
& 526 & 26.35 & 0.765 & 0.288
& 495 & 28.01 & 0.873 & 0.315 \\
PUP-3DGS-20\%~\citep{Hanson2025PUP3DGS}
& 315 & \rankthree{22.60} & 0.809 & 0.228
& 525 & \rankthree{26.67} & \ranktwo{0.790} & \ranktwo{0.257}
& 495 & \rankthree{29.72} & \rankthree{0.905} & \rankone{0.248} \\
Trimming the Fat-20\%~\citep{Ali2024TrimmingFat}
& 315 & 22.50 & 0.819 & 0.232
& 524 & 26.49 & 0.781 & 0.280
& 494 & 29.08 & 0.900 & 0.272 \\
MesonGS-20\%~\citep{Xie2024MesonGS}
& 314 & 20.70 & \rankthree{0.822} & \ranktwo{0.207}
& 527 & 25.01 & 0.776 & \rankthree{0.262}
& 496 & 28.70 & 0.897 & \rankthree{0.262} \\
3DGS+GHAP-20\%~\citep{Wang2025GHAP}
& 314 & \ranktwo{23.62} & \ranktwo{0.835} & \rankthree{0.212}
& 527 & \ranktwo{26.97} & \rankthree{0.788} & 0.275
& 496 & \ranktwo{29.86} & \rankone{0.907} & \ranktwo{0.252} \\
3DGS-MCMC~\citep{Kheradmand2024MCMC}
& 315 & 22.14 & 0.779 & 0.282
& 263 & 25.96 & 0.763 & 0.309
& 496 & 28.98 & 0.885 & 0.298 \\
\textbf{Ours}
& \textbf{500} & \rankone{24.05} & \rankone{0.857} & \rankone{0.178}
& \textbf{500} & \rankone{27.61} & \rankone{0.815} & \rankone{0.238}
& \textbf{500} & \rankone{30.04} & \ranktwo{0.906} & \rankthree{0.262} \\
\bottomrule
\end{tabular}
}

\caption{
Compact-model comparison with external results from GHAP~\citep{Wang2025GHAP}. Mip-NeRF360 averages all nine scenes here, including Flowers and Treehill. Ours uses a 500k cap; external counts and training protocols differ. Counts are in thousands. Colors denote {\setlength\fboxsep{1pt}\colorbox{best}{1st}}, {\setlength\fboxsep{1pt}\colorbox{second}{2nd}}, and {\setlength\fboxsep{1pt}\colorbox{third}{3rd}} among compact rows; Vanilla is excluded.
}
\label{tab:compact_regime_main}
\end{table*}

Our controller produces compact models directly during training. The matched-cap comparisons in Table~\ref{tab:sota_comparison} and the matched-controller ablation in Table~\ref{tab:controller_ablation} provide the controlled evidence for its quality gains; Table~\ref{tab:compact_regime_main} provides additional context across training pipelines.

\subsection{Discussion and Limitations}
\label{subsec:discussion_limitations}

\paragraph{Transfer beyond MCMC.}
The first-variation score measures local mass utility at existing
parent sites, while the backend determines how new capacity is
realized. Replacing Vanilla 3DGS's position-gradient criterion
with $q_i$, while retaining native cloning, splitting, and pruning,
provides empirical evidence of transfer beyond MCMC. Protocols and results are provided in Appendix~\ref{app:vanilla_transfer}.

\paragraph{Limitations and future work.} While the first-variation score provides a principled surrogate for mass allocation, several natural frontiers remain. First, as a scalar mass derivative evaluated at existing sites, $q_i$ dictates \emph{where} to allocate capacity rather than parameterizing the optimal geometric displacement of split children; extending this to joint geometry-mass proposal scoring is an exciting direction. Second, because the score tracks photometric residuals, severe cross-view inconsistencies (e.g., auto-exposure in unconstrained captures) can occasionally misguide allocation at high capacities (Appendices~\ref{app:train_failure_case} and~\ref{app:structural_cutoff}). Coupling variational density control with exposure- or uncertainty-aware models presents a promising solution. Finally, adaptive score schedules that decay update frequency as geometry stabilizes could further amortize the auxiliary backward-pass overhead (Appendix~\ref{app:computation_tradeoff}).
\section{Conclusion} 
We reframed 3DGS structural control from a collection of heuristics into a formal mass-allocation problem. By evaluating the first variation of the multiview reconstruction risk, we replaced heuristic cloning and pruning with a  principled birth--death controller. Our experiments expose the fundamental limitation of standard 3DGS: proxy signals succeed when excess capacity hides their inefficiencies. When the primitive budget is constrained, querying the objective's mass derivative is essential to preserve complex geometry and thin structures. This demonstrates that 3DGS structural updates belong in the same rigorous empirical risk minimization framework as continuous parameters. Moving forward, future compact and dynamic radiance fields can bypass heuristic bottlenecks by letting the reconstruction risk govern structural growth.

\begin{ack}
\paragraph{Acknowledgments.}
The authors thank the anonymous reviewers and the Area Chair for their helpful feedback. This work is supported by grant VUNI.2526.CAIR.04.

\paragraph{Competing interests.}
The authors declare no competing interests.
\end{ack}%
\bibliographystyle{abbrvnat}
\bibliography{references}

\clearpage
\appendix
\raggedbottom
\setcounter{topnumber}{4}
\setcounter{bottomnumber}{3}
\setcounter{totalnumber}{6}
\renewcommand{\topfraction}{0.95}
\renewcommand{\bottomfraction}{0.85}
\renewcommand{\textfraction}{0.05}
\renewcommand{\floatpagefraction}{0.75}
\makeatletter
\setlength{\@fptop}{0pt}
\setlength{\@fpsep}{12pt plus 2pt minus 2pt}
\setlength{\@fpbot}{0pt plus 1fil}
\makeatother

\section{Deferred Proofs from Section \ref{sec:method}}
\label{app:deferred_proofs}

\subsection{Proof of Proposition \ref{prop:qi_first_variation}}
\begin{proof}
Fix $i\in\{1,\dots,M_t\}$.  For each $n\in\mathcal B_t$, define
\[
\varphi_n(\mathbf m)
:=
F_n\bigl(f_{n,t}(\cdot;\mathbf m)\bigr)
=
\ell_{\mathrm{img}}\!\bigl(
\Phi(\mathcal G_t^{\mathbf m},\pi_n),I_n
\bigr).
\]
By assumption, $F_n:L^2(\Omega;\mathbb R^3)\to\mathbb R$ is Fr\'echet
differentiable at $f_{n,t}(\cdot;\mathbf 1)$, with differential
$DF_n[f_{n,t}(\cdot;\mathbf 1)]$.  Moreover,
$f_{n,t}(\cdot;\mathbf m)$ is differentiable in $\mathbf m$ at
$\mathbf m=\mathbf 1$, so its $i$-th partial derivative
\[
\xi_{i,n}
:=
\left.
\frac{\partial f_{n,t}(\cdot;\mathbf m)}{\partial m_i}
\right|_{\mathbf m=\mathbf 1}
\in L^2(\Omega;\mathbb R^3)
\]
is well-defined.  The chain rule in Banach spaces gives
\[
\left.
\frac{\partial \varphi_n(\mathbf m)}{\partial m_i}
\right|_{\mathbf m=\mathbf 1}
=
DF_n\!\bigl[f_{n,t}(\cdot;\mathbf 1)\bigr](\xi_{i,n}).
\]
By Definition~\ref{def:first_variation}, the Riesz representative of
$DF_n$ is $\delta F_n[f_{n,t}(\cdot;\mathbf 1)]$, hence
\[
DF_n\!\bigl[f_{n,t}(\cdot;\mathbf 1)\bigr](\xi_{i,n})
=
\int_\Omega
\Bigl\langle
\xi_{i,n}(x),\,
\delta F_n\!\bigl[f_{n,t}(\cdot;\mathbf 1)\bigr](x)
\Bigr\rangle\,dx .
\]
Differentiating the finite sum in~\eqref{eq:method_first_variation_score}
term by term gives
\[
\left.
\frac{\partial \widehat F_t(\mathbf m)}
{\partial m_i}
\right|_{\mathbf m=\mathbf 1}
=
\frac{1}{B}\sum_{n\in\mathcal B_t}
\int_\Omega
\Bigl\langle
\delta F_n\!\bigl[f_{n,t}(\cdot;\mathbf 1)\bigr](x),\;
\left.
\frac{\partial f_{n,t}(x;\mathbf m)}{\partial m_i}
\right|_{\mathbf m=\mathbf 1}
\Bigr\rangle\,dx .
\]
Multiplying by $-1$ yields the claimed expression for $(q_t)_i$.
\end{proof}

\subsection{Proof of Proposition \ref{prop:descent_lemma}}
\begin{proof}
Let $x:=\mathbf 1$ and $y:=\mathbf 1+\delta\mathbf m$.  Since $F$ is
differentiable on a neighbourhood containing the segment
$\{x+\tau(y-x):\tau\in[0,1]\}$, the fundamental theorem of calculus yields
\[
F(y)-F(x)
=
\int_0^1
\Bigl\langle
\nabla F\bigl(x+\tau(y-x)\bigr),\,y-x
\Bigr\rangle\,d\tau .
\]
Add and subtract $\nabla F(x)$ inside the integrand:
\[
F(y)-F(x)
=
\langle \nabla F(x),\,y-x\rangle
+
\int_0^1
\Bigl\langle
\nabla F\bigl(x+\tau(y-x)\bigr)-\nabla F(x),\,y-x
\Bigr\rangle\,d\tau .
\]
By Cauchy--Schwarz and $L_{\mathrm{sm}}$-smoothness,
\[
\Bigl|
\Bigl\langle
\nabla F\bigl(x+\tau(y-x)\bigr)-\nabla F(x),\,y-x
\Bigr\rangle
\Bigr|
\le
L_{\mathrm{sm}}\,\tau\,\|y-x\|_2^2 .
\]
Integrating $\tau$ from $0$ to $1$ gives
\[
F(y)
\le
F(x)+\langle \nabla F(x),\,y-x\rangle
+
\frac{L_{\mathrm{sm}}}{2}\|y-x\|_2^2 .
\]
Finally, $\nabla F(\mathbf 1)=-q_t$ by definition of $q_t$, and
$y-x=\delta\mathbf m$, which yields the stated bound.
\end{proof}

\subsection{Proof of Proposition \ref{prop:low_opacity_bound}}
\begin{proof}
Fix a view $n$ and a pixel $x\in\Omega_{\mathrm{pix}}$.  For readability we
suppress the dependence on $(n,x)$ and write
\[
a_k:=a_k(x)=\alpha_k G_k(x)\in[0,1],
\qquad
T_i:=\prod_{k\prec_n i}(1-a_k)\in[0,1],
\qquad
w_i:=T_i a_i .
\]
Define tail colors recursively by choosing a terminal color
$U_{M_t+1}(x)\in[0,1]^3$ and setting
\[
U_k
:=
a_k\,c_k(\mathbf d_n)
+
(1-a_k)\,U_{k+1},
\qquad
k=M_t,\dots,1 .
\]
Then $\hat I_n(x)=U_1$, and $U_k\in[0,1]^3$ for all $k$ by backward induction.

Let $\hat I_n^{(-i)}$ denote the image obtained by setting $a_i=0$ at every
pixel.  At the fixed pixel $x$, the corresponding tail recursion satisfies
$U_k^{(-i)}=U_k$ for $k\ge i+1$ and
\[
U_i^{(-i)}
=
U_{i+1}.
\]
Because the transmittance before $i$ depends only on Gaussians in front of
$i$, it is unchanged by setting $a_i=0$.  Decomposing the pixel value into the
front contribution and the tail starting at $i$ gives
\[
\hat I_n(x)
=
\sum_{k\prec_n i} w_k\,c_k(\mathbf d_n)
+
T_i U_i,
\qquad
\hat I_n^{(-i)}(x)
=
\sum_{k\prec_n i} w_k\,c_k(\mathbf d_n)
+
T_i U_i^{(-i)} .
\]
Subtracting and using
$U_i-U_i^{(-i)}=U_i-U_{i+1}=a_i(c_i(\mathbf d_n)-U_{i+1})$ yields
\[
\hat I_n(x)-\hat I_n^{(-i)}(x)
=
T_i a_i
\bigl(c_i(\mathbf d_n)-U_{i+1}\bigr)
=
w_i
\bigl(c_i(\mathbf d_n)-U_{i+1}\bigr).
\]
Since $G_i(x)\le 1$ and $T_i\le 1$, we have
$w_i=T_i a_i\le a_i\le\alpha_i$.  Since both
$c_i(\mathbf d_n)$ and $U_{i+1}$ lie in $[0,1]^3$,
\[
\frac13
\bigl\|
\hat I_n(x)-\hat I_n^{(-i)}(x)
\bigr\|_1
\le
w_i
\le
\alpha_i .
\]
Averaging over $x\in\Omega_{\mathrm{pix}}$ gives
$\|\hat I_n-\hat I_n^{(-i)}\|_{1,\mathrm{mean}}\le\alpha_i$.

Finally, if $\ell_{\mathrm{img}}(\cdot,I_n)$ is $L_\ell$-Lipschitz with
respect to $\|\cdot\|_{1,\mathrm{mean}}$, then
\[
\bigl|
\ell_{\mathrm{img}}(\hat I_n,I_n)
-
\ell_{\mathrm{img}}(\hat I_n^{(-i)},I_n)
\bigr|
\le
L_\ell\,
\|\hat I_n-\hat I_n^{(-i)}\|_{1,\mathrm{mean}}
\le
L_\ell\,\alpha_i ,
\]
as claimed.
\end{proof}

\section{Implementation and Reproducibility}
\label{app:full_algorithm}
\label{app:reproducibility}

This section gives the default hyperparameters used in all experiments. The main paper analyzes the variational birth--death decision; the finite relocation, optimizer, regularizers, and event schedule are inherited from the 3DGS-MCMC backend and are kept fixed in the matched comparisons.

\begin{table}[!htbp]
\caption{
Experimental configuration and default hyperparameters.
}
\label{tab:app_config}
\centering
\small
\setlength{\tabcolsep}{6pt}
\begin{tabularx}{\linewidth}{l X}
\toprule
\textbf{Item} & \textbf{Value} \\
\midrule
Training iterations & 30{,}000 \\
Gaussian caps & 100k, 200k, 500k, 1M, 2M \\
Initialization / seed & SfM or random initialization, seed 0 \\
Controller interval $\Delta$ & 40 iterations \\
Score batch size $B$ & 16 views \\
EMA coefficient $\rho$ & 0.9 \\
Clipping quantile $\tau$ & 0.99 absolute quantile \\
Dead-opacity threshold $\alpha_{\mathrm{death}}$ & 0.015 \\
Birth-site weights & $(\bar q_t)_i$ \\
Death policy & Lowest $(\bar q_t)_i$ within the low-opacity pool \\
Jitter scale $\eta$ & 1.0 \\
Photometric loss & $(1-\lambda_{\mathrm{dssim}})\ell_1 + \lambda_{\mathrm{dssim}}(1-\mathrm{SSIM})$ \\
$\lambda_{\mathrm{dssim}}$ & 0.2 \\
Opacity regularizer weight $\lambda_\alpha$ & 0.01 \\
Scale regularizer weight $\lambda_s$ & 0.01 \\
Optimizer & 3DGS Adam parameter groups with unchanged learning rates \\
GPU & NVIDIA RTX 4090 \\
CUDA / PyTorch & CUDA 12.8 / PyTorch 2.11.0 \\
Evaluation metrics & PSNR, SSIM, LPIPS \\
\bottomrule
\end{tabularx}
\end{table}

\section{Matched 3DGS-MCMC Drop-in Comparison}
\label{app:dropin_mcmc}

The main text uses the matched 3DGS-MCMC comparison to isolate the effect of the structural statistic.  Table~\ref{tab:app_mcmc_gains} keeps the renderer, loss, optimizer, relocation kernel, schedule, initialization type, and Gaussian cap fixed; only the controller statistic changes from opacity to the first-variation score.  The gains are largest in the low-budget regime, where the birth--death allocation decision is most constrained.  At larger caps, the effect becomes dataset- and metric-dependent, consistent with the main-paper interpretation that additional capacity can partially mask allocation errors.

\begin{table}[!htbp]
\caption{
Controlled drop-in comparison against our matched 3DGS-MCMC rerun. Positive deltas indicate improvement: $\Delta$PSNR = Ours $-$ MCMC, $\Delta$SSIM = Ours $-$
MCMC, and $\Delta$LPIPS = MCMC $-$ Ours.
}
\label{tab:app_mcmc_gains}
\centering
\small
\setlength{\tabcolsep}{4.5pt}
\begin{adjustbox}{width=\linewidth}
\begin{tabular}{l r ccc ccc ccc}
\toprule
\multirow{2}{*}{\textbf{Dataset}}
& \multirow{2}{*}{\textbf{Cap}}
& \multicolumn{3}{c}{\textbf{3DGS-MCMC rerun}}
& \multicolumn{3}{c}{\textbf{Ours}}
& \multicolumn{3}{c}{\textbf{Gain}} \\
\cmidrule(lr){3-5} \cmidrule(lr){6-8} \cmidrule(lr){9-11}
& & PSNR & SSIM & LPIPS & PSNR & SSIM & LPIPS & $\Delta$PSNR & $\Delta$SSIM & $\Delta$LPIPS \\
\midrule
\multirow{5}{*}{Mip-NeRF360} & 100k & 27.61 & 0.814 & 0.281 & 27.94 & 0.818 & 0.276 & +0.33 & +0.004 & +0.005 \\
 & 200k & 28.04 & 0.835 & 0.254 & 28.37 & 0.841 & 0.249 & +0.33 & +0.005 & +0.006 \\
 & 500k & 28.89 & 0.864 & 0.209 & 29.11 & 0.869 & 0.203 & +0.22 & +0.005 & +0.006 \\
 & 1M & 29.35 & 0.878 & 0.185 & 29.49 & 0.883 & 0.178 & +0.14 & +0.005 & +0.006 \\
 & 2M & 29.60 & 0.885 & 0.168 & 29.66 & 0.890 & 0.163 & +0.07 & +0.004 & +0.005 \\
\addlinespace
\multirow{5}{*}{Tanks \& Temples} & 100k & 23.21 & 0.826 & 0.232 & 23.60 & 0.831 & 0.225 & +0.39 & +0.005 & +0.006 \\
 & 200k & 23.45 & 0.833 & 0.221 & 23.68 & 0.836 & 0.216 & +0.23 & +0.004 & +0.005 \\
 & 500k & 24.06 & 0.854 & 0.184 & 24.05 & 0.857 & 0.178 & -0.01 & +0.003 & +0.006 \\
 & 1M & 24.43 & 0.867 & 0.160 & 24.17 & 0.865 & 0.157 & -0.26 & -0.001 & +0.003 \\
 & 2M & 24.66 & 0.875 & 0.141 & 24.21 & 0.870 & 0.143 & -0.45 & -0.005 & -0.002 \\
\addlinespace
\multirow{5}{*}{DeepBlending} & 100k & 29.14 & 0.891 & 0.305 & 29.45 & 0.892 & 0.303 & +0.32 & +0.001 & +0.002 \\
 & 200k & 29.71 & 0.899 & 0.283 & 29.72 & 0.899 & 0.284 & +0.01 & -0.000 & -0.001 \\
 & 500k & 29.68 & 0.905 & 0.262 & 30.04 & 0.906 & 0.262 & +0.35 & +0.001 & -0.001 \\
 & 1M & 29.94 & 0.909 & 0.249 & 30.14 & 0.911 & 0.248 & +0.20 & +0.002 & +0.000 \\
 & 2M & 30.05 & 0.910 & 0.239 & 30.19 & 0.912 & 0.240 & +0.14 & +0.002 & -0.001 \\
\bottomrule
\end{tabular}
\end{adjustbox}
\end{table}

\section{Transfer beyond 3DGS-MCMC}
\label{app:vanilla_transfer}

\paragraph{Local utility and spatial realization.}
The first-variation score measures the utility of rescaling a Gaussian's
contribution at its current support, independently of the MCMC transition.
The backend realizes the selected capacity through co-located relocation
or Vanilla cloning and spatial splitting~\citep{Kerbl2023GaussianSplatting}.
This separates parent selection from spatial realization: the score does
not determine a split child's displacement.

\paragraph{Vanilla protocol.}
We replace Vanilla's position-gradient criterion with $q_i$, retaining
native splitting geometry, pruning thresholds, and the 30k training
schedule within each comparison. The uncapped experiment selects parents with
$q_i>\tau$ on three Mip-NeRF360 scenes. The threshold $\tau$ is a
tunable hyperparameter, set to $10^{-7}$ for this experiment. Separately, we evaluate
parent selection with Vanilla clone/split geometry at a 500k cap on
five scenes, and native clone/split/prune at a 100k cap on Bicycle.

\begin{table}[!htbp]
\centering
\caption{Uncapped Vanilla 3DGS transfer. Each quality entry reports
PSNR$\uparrow$ / SSIM$\uparrow$ / LPIPS$\downarrow$; final Gaussian
counts are in millions.}
\label{tab:vanilla_uncapped}
\small
\setlength{\tabcolsep}{4pt}
\renewcommand{\arraystretch}{1.08}
\begin{tabular}{@{}l cc rr@{}}
\toprule
& \multicolumn{2}{c}{Reconstruction quality} & \multicolumn{2}{c}{Final \#G (M)} \\
\cmidrule(lr){2-3}\cmidrule(l){4-5}
Scene & Vanilla & Vanilla + $q_i$ & Vanilla & + $q_i$ \\
\midrule
Bicycle & 25.13 / 0.747 / 0.242 & 25.29 / 0.774 / 0.202 & 4.769 & 3.271 \\
Bonsai  & 32.46 / 0.947 / 0.180 & 32.67 / 0.952 / 0.151 & 2.167 & 1.064 \\
Flowers & 21.36 / 0.589 / 0.359 & 21.61 / 0.627 / 0.278 & 3.206 & 2.850 \\
\bottomrule
\end{tabular}
\end{table}

\begin{table}[!htbp]
\centering
\caption{Capped Vanilla transfer. Entries report PSNR$\uparrow$ /
SSIM$\uparrow$ / LPIPS$\downarrow$; the mean uses the five 500k
scenes only. Final Bicycle/100k populations are 99,812 (position
gradient) and 99,910 ($q_i$).}
\label{tab:vanilla_capped}
\small
\setlength{\tabcolsep}{7pt}
\renewcommand{\arraystretch}{1.08}
\begin{tabular}{@{}l cc@{}}
\toprule
Scene & Position gradient & First-variation score $q_i$ \\
\midrule
\multicolumn{3}{c}{500k parent-selection diagnostic} \\
\midrule
Bonsai  & 29.81 / 0.908 / 0.277 & 30.20 / 0.912 / 0.266 \\
Counter & 27.67 / 0.874 / 0.267 & 27.95 / 0.877 / 0.261 \\
Garden  & 24.51 / 0.689 / 0.361 & 24.58 / 0.691 / 0.357 \\
Kitchen & 28.05 / 0.878 / 0.218 & 28.94 / 0.883 / 0.210 \\
Room    & 30.17 / 0.894 / 0.277 & 30.67 / 0.897 / 0.273 \\
\midrule
Mean    & 28.04 / 0.849 / 0.280 & 28.47 / 0.852 / 0.274 \\
\midrule
\multicolumn{3}{c}{100k native clone/split/prune} \\
\midrule
Bicycle & 23.12 / 0.580 / 0.445 & 23.03 / 0.605 / 0.414 \\
\bottomrule
\end{tabular}
\end{table}

\paragraph{Results.}
Table~\ref{tab:vanilla_uncapped} shows improved quality with fewer
final Gaussians; it is not a matched-population comparison.
Table~\ref{tab:vanilla_capped} shows gains across all metrics on the
five 500k scenes. At 100k on Bicycle, SSIM and LPIPS improve while
PSNR decreases slightly. This supports transfer, not uniform superiority.

\paragraph{Scope and future work.}
A scalar mass derivative does not evaluate the full geometric change
of a finite split or the child's eventual reconstruction benefit.
The Bicycle trade-off is consistent with this distinction, but does
not establish displacement as its cause. Future work could score finite proposals jointly with their geometry,
rather than parent mass alone.

\section{High-Capacity Overfitting and Structural-Update Cutoff}
\label{app:structural_cutoff}

To test the effect of late structural updates, we stop birth--death
relocations at iteration 15k while continuing continuous parameter
optimization until 30k. Both variants reach their respective Gaussian
caps before the cutoff, and other settings are held fixed within this
diagnostic. Table~\ref{tab:structural_cutoff} reports the Train results.

\begin{table}[!htbp]
\centering
\caption{Structural-update cutoff on Tanks \& Temples Train.
Continuous optimization runs for 30k iterations in both variants.
These are within-diagnostic comparisons, not replacements for the
default benchmark results.}
\label{tab:structural_cutoff}
\small
\setlength{\tabcolsep}{7pt}
\renewcommand{\arraystretch}{1.08}
\begin{tabular}{@{}ll rrr@{}}
\toprule
Cap & Structural updates & PSNR$\uparrow$ & SSIM$\uparrow$ & LPIPS$\downarrow$ \\
\midrule
500k & Full updates & 21.88 & 0.820 & 0.216 \\
500k & Stop at 15k & 22.36 & 0.823 & 0.217 \\
\midrule
1M & Full updates & 22.06 & 0.835 & 0.190 \\
1M & Stop at 15k & 22.24 & 0.831 & 0.196 \\
\bottomrule
\end{tabular}
\end{table}

Stopping structural updates improves PSNR in these runs, while some
SSIM and LPIPS values worsen. This supports a contribution from late
structural updates, but does not establish a complete remedy or isolate
photometric variation as the sole cause. We retain the default schedule
for the main benchmark and treat the cutoff as a diagnostic variant.
An adaptive schedule based on cross-view score reliability is a direction
for future work.

\section{Extended Comparison to Compact 3DGS}
\label{app:compact_regime}

As introduced in Section~\ref{subsec:compact_regime}, our WFR-derived controller intrinsically prevents the primitive bloat typical of standard 3DGS. Rather than training a massive, computationally expensive model only to prune it post-hoc, our method allocates finite capacity efficiently from scratch. This section provides an extended comparison against recent compaction baselines in the extreme low-primitive regime ($\sim$10\% retention) across all three datasets.

\textbf{Protocol distinction.}
Table~\ref{tab:compact_regime_10pct} contextualizes our results against representative compact 3DGS baselines reported by GHAP~\cite{Wang2025GHAP}. The experimental protocols differ fundamentally: baselines like LightGaussian~\citep{Fan2024LightGaussian}, PUP-3DGS~\citep{Hanson2025PUP3DGS}, and MesonGS~\citep{Xie2024MesonGS} prune and compress a fully trained backbone. In contrast, our method trains directly from scratch under a strict Gaussian cap. We include this comparison to demonstrate a key point: principled, first-variation allocation during training achieves highly competitive compactness without the wasted overhead of an over-parameterized initial phase.

\begin{table*}[!htbp]
\vspace{-0.75em}
\centering
\caption{
Extreme low-primitive comparison ($\sim$10\% retention) on Mip-NeRF360, Tanks~\&~Temples, and DeepBlending. External compact-baseline values are from GHAP~\citep{Wang2025GHAP}. Vanilla 3DGS is shown only as a full-model reference and is not included in the rank highlighting. Rank colors are computed only among the low-primitive rows: {\setlength\fboxsep{1pt}\colorbox{best}{best}}, {\setlength\fboxsep{1pt}\colorbox{second}{second best}} and {\setlength\fboxsep{1pt}\colorbox{third}{third best}}. Higher PSNR/SSIM and lower LPIPS are better. All Gaussian counts are in thousands.
}
\vspace{0.75em}
\label{tab:compact_regime_10pct}
\resizebox{\textwidth}{!}{
\begin{tabular}{l r c c c r c c c r c c c}
\toprule
\multirow{2}{*}{\textbf{Method}} &
\multicolumn{4}{c}{\textbf{Tanks \& Temples}} &
\multicolumn{4}{c}{\textbf{Mip-NeRF360}} &
\multicolumn{4}{c}{\textbf{DeepBlending}} \\
\cmidrule(lr){2-5}
\cmidrule(lr){6-9}
\cmidrule(lr){10-13}
&
\#G & PSNR$\uparrow$ & SSIM$\uparrow$ & LPIPS$\downarrow$ &
\#G & PSNR$\uparrow$ & SSIM$\uparrow$ & LPIPS$\downarrow$ &
\#G & PSNR$\uparrow$ & SSIM$\uparrow$ & LPIPS$\downarrow$ \\
\midrule
Vanilla 3DGS~\citep{Kerbl2023GaussianSplatting}
& 1577 & 23.79 & 0.853 & 0.169
& 2627 & 27.55 & 0.813 & 0.221
& 2475 & 29.82 & 0.907 & 0.238 \\
\midrule
LightGaussian~\citep{Fan2024LightGaussian}
& 158 & 22.11 & 0.756 & 0.306
& 263 & 25.67 & 0.735 & 0.331
& 248 & 28.01 & 0.869 & 0.327 \\
PUP-3DGS~\citep{Hanson2025PUP3DGS}
& 158 & 21.52 & 0.767 & 0.280
& 262 & 25.33 & 0.753 & \rankthree{0.309}
& 248 & \rankthree{29.15} & 0.895 & \ranktwo{0.274} \\
Trimming the Fat~\citep{Ali2024TrimmingFat}
& 156 & 21.54 & 0.776 & 0.293
& 263 & 25.26 & 0.731 & 0.343
& 247 & 28.06 & 0.887 & 0.302 \\
MesonGS~\citep{Xie2024MesonGS}
& 157 & 20.71 & \rankthree{0.811} & \rankone{0.208}
& 263 & 24.92 & \ranktwo{0.773} & \rankone{0.264}
& 248 & 28.69 & \rankthree{0.896} & \rankone{0.264} \\
3DGS+GHAP~\citep{Wang2025GHAP}
& 157 & \ranktwo{23.31} & \ranktwo{0.818} & \rankthree{0.242}
& 263 & \ranktwo{26.40} & \rankthree{0.764} & 0.314
& 248 & \ranktwo{29.65} & \rankone{0.905} & \rankone{0.264} \\
3DGS-MCMC~\citep{Kheradmand2024MCMC}
& 157 & \rankthree{22.14} & 0.779 & 0.282
& 263 & \rankthree{25.96} & 0.763 & \rankthree{0.309}
& 248 & 28.98 & 0.885 & 0.298 \\
\textbf{Ours}
& \textbf{200} & \rankone{23.68} & \rankone{0.836} & \ranktwo{0.216}
& \textbf{200} & \rankone{26.93} & \rankone{0.783} & \ranktwo{0.287}
& \textbf{200} & \rankone{29.72} & \ranktwo{0.899} & \rankthree{0.284} \\
\bottomrule
\end{tabular}
}
\end{table*}

\textbf{Extreme low-primitive comparison.} On the Mip-NeRF360 benchmark, our method uses significantly fewer primitives than the baselines (200k vs. $\sim$263k) yet obtains the highest PSNR and SSIM among the low-primitive models. This trend holds across Tanks~\&~Temples and DeepBlending, where our model delivers the best PSNR and competitive SSIM. While post-training attribute transformation methods (e.g., MesonGS) achieve lower LPIPS in specific settings, our controller establishes a highly competitive distortion-oriented reconstruction frontier straight out of training.

\textbf{Higher-budget comparison.} The equivalent comparison for the 500k-budget regime ($\sim$20\% retention) is provided in Section~\ref{subsec:compact_regime} of the main text. Together, these extreme and higher-budget comparisons reinforce the central thesis of the paper: when the representation is constrained to a compact regime, the WFR mass derivative provides a vastly superior structural statistic for capacity allocation than standard heuristics.
\section{Discussion}
\label{app:discussion}

\subsection{Mass score and opacity-gradient parameterization}
\label{app:mass_opacity_gradient}

The proposed mass-coordinate score is closely related to standard opacity gradients, but it provides a more appropriate coordinate space for structural allocation.

In the frozen score pass, the auxiliary coordinate \(m_i\) modulates the renderer through the opacity footprint: \(a_i(x;\mathbf m)=m_i\alpha_iG_i(x)\). Suppressing the view index and batch average for readability, the chain rule yields a direct equivalence to the post-activation opacity derivative:
\begin{equation*}
q_i
=
\left.
-\frac{\partial \widehat F_t}{\partial m_i}
\right|_{\mathbf m=\mathbf 1}
=
-\sum_x
\frac{\partial \widehat F_t}{\partial a_i(x)}
\alpha_iG_i(x)
=
-\alpha_i
\frac{\partial \widehat F_t}{\partial \alpha_i}.
\end{equation*}

This scaling by \(\alpha_i\) explains the practical stability of the mass-coordinate score. The raw opacity derivative \(\partial \widehat F_t / \partial \alpha_i\) evaluates a scalar perturbation without accounting for the primitive's current presence in the scene. In practice, raw gradients often spike on nearly invisible "floaters" attempting to fit high-frequency noise. By measuring the relative rescaling of the primitive's mass, the factor \(\alpha_i\) acts as an objective-level confidence weight. A nearly transparent primitive with a noisy gradient is correctly suppressed, forcing the controller to prioritize structurally established Gaussians for birth sites. For death decisions, the logic inverts cleanly: we first filter for safe slots via the low-opacity pool, then use the score to cull the slots whose removal will least impact the reconstruction loss.

Crucially, the mass score also bypasses a hidden artifact in standard continuous optimization. 3DGS typically optimizes opacity via a logit \(\zeta_i\), where \(\alpha_i=\sigma(\zeta_i)\). The gradient is therefore:
\begin{equation*}
\frac{\partial \widehat F_t}{\partial \zeta_i}
=
\alpha_i(1-\alpha_i)
\frac{\partial \widehat F_t}{\partial \alpha_i}.
\end{equation*}
Because of the \((1-\alpha_i)\) factor, logit gradients saturate and vanish for highly opaque, solid Gaussians (\(\alpha_i \approx 1\)). If a density controller relied on the optimizer's internal logit gradient to identify structural candidates, it would inherently penalize the most established, reliable primitives in the scene. 

Because our mass score \(q_i\) is derived directly from the linear mass coordinate \(m_i\), it measures the true first-order descent direction, isolated from the saturation artifacts of the activation function. We therefore use \(q_i\) exclusively for discrete birth--death selection, leaving continuous opacity optimization to the inherited 3DGS-MCMC backend.

\subsection{High-budget behavior on Tanks \& Temples Train}
\label{app:train_failure_case}

At larger capacities, our method occasionally trails the opacity-driven 3DGS-MCMC baseline on Tanks \& Temples. Table~\ref{tab:app_train_photometric_diagnostic} isolates the main source of this variance. The high-budget regression is concentrated on \emph{Train}: at 2M Gaussians, \emph{Train} has 13 severe views where our method trails by more than 1 dB, while \emph{Truck} has none. Mip-NeRF360 remains stable overall at the same capacity.

\begin{table}[!htbp]
\centering
\caption{
Diagnostic evidence for the high-budget Tanks \& Temples Train behavior. \(\Delta\)PSNR denotes Ours minus the matched 3DGS-MCMC rerun. The exposure-compensation row reruns both 3DGS-MCMC and Ours under the same exposure-compensation protocol; for Train at 2M, MCMC/Ours PSNR changes from \(22.95/22.10\) to \(24.95/24.65\). This is a paired diagnostic, not a replacement for the default main-table protocol. Severe views denote \(\Delta\)PSNR \(<-1\) dB.
}
\label{tab:app_train_photometric_diagnostic}
\small
\setlength{\tabcolsep}{4.5pt}
\resizebox{\linewidth}{!}{%
\begin{tabular}{l c c c c c}
\toprule
\textbf{Setting} &
\textbf{Protocol} &
\textbf{Mean \(\Delta\)PSNR} &
\textbf{Severe views} &
\textbf{View-luma std.} &
\textbf{Mean bright-pixel frac.} \\
\midrule
Train, 2M & Default & \(-0.85\) & \(13/38\) & \(0.088\) & \(0.094\) \\
Train, 2M & Exp. comp. & \(-0.29\) & \(6/38\) & -- & -- \\
Truck, 2M & Default & \(-0.05\) & \(0/32\) & \(0.061\) & \(0.048\) \\
Mip-NeRF360 avg., 2M & Default & \(+0.06\) & \(3/246\) & \(0.062\) & \(0.031\) \\
\bottomrule
\end{tabular}%
}
\end{table}

This pattern is consistent with the capture conditions. Mip-NeRF360 was captured with fixed exposure and controlled lighting to reduce photometric variation~\citep{barron2022mipnerf360}. Tanks \& Temples is a realistic outside-the-lab benchmark, and \emph{Train} was captured under sunny illumination with automatic ISO to handle varying exposure~\citep{knapitsch2017tanks}. Our image statistics confirm this trend: \emph{Train} exhibits higher view-level brightness variation and a larger bright-pixel fraction than \emph{Truck} or the Mip-NeRF360 average.

This exposes a useful boundary of residual-driven density control. The first-variation score directly follows the empirical reconstruction loss. This is highly desirable when residuals correspond to missing geometry, underrepresented texture, or thin structures. However, at high capacity on \emph{Train}, many easy structural errors have already been resolved. The remaining residuals are heavily weighted toward exposure changes, saturated sky regions, shadows, and view-dependent train surfaces. In that regime, a large image residual is not always reliable birth-site evidence. Robust radiance-field literature makes a related observation: lighting variation, shadows, and other non-persistent effects produce view-dependent floaters when treated as ordinary reconstruction residuals~\citep{Sabour2023RobustNeRF}. Recent image-formation work similarly shows that photometric variation can be absorbed as geometric artifacts when no compensation model is present~\citep{Deutsch2026PPISP}.

\begin{figure}[!htbp]
\centering
\includegraphics[width=\linewidth]{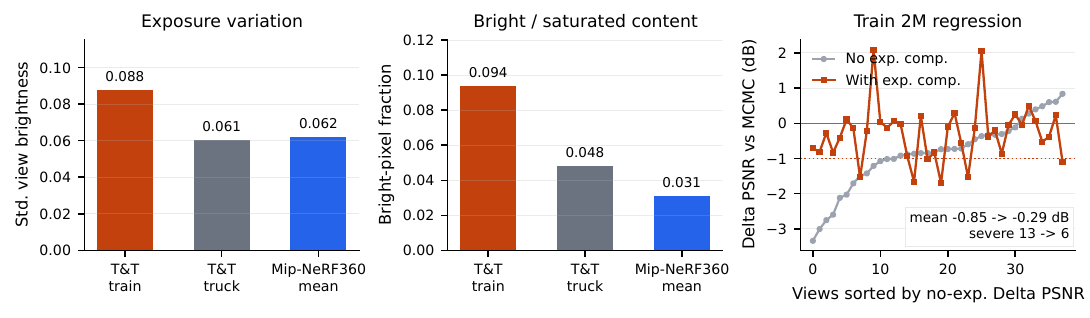}
\caption{
Photometric diagnostic for the high-budget \emph{Train} case. Left and middle: \emph{Train} has higher view-level brightness variation and a larger bright-pixel fraction than \emph{Truck} and the Mip-NeRF360 average. Right: paired exposure-compensation reruns for both 3DGS-MCMC and Ours reduce the mean Train/SfM/2M gap from \(-0.85\) dB to \(-0.29\) dB and reduce severe views from 13 to 6. This supports the interpretation that the high-budget regression is tied to residual reliability under photometric variation, not to a general breakdown of first-variation allocation.
}
\label{fig:app_tnt_train_exposure_diagnostic}
\end{figure}

The exposure-compensation diagnostic in Figure~\ref{fig:app_tnt_train_exposure_diagnostic} narrows the performance gap, supporting the interpretation that photometric inconsistency contributes to the high-budget regression. Once the residual is made photometrically consistent, our method gains \(+2.55\) dB---significantly outpacing MCMC's \(+2.00\) dB recovery. (We retain this as a diagnostic because the official 3DGS exposure setting alters the standard train/test split, breaking comparability with standard benchmarks~\citep{Kerbl2023GaussianSplattingCode}).

\subsection{Compute trade-off and scalability}
\label{app:computation_tradeoff}

Replacing heuristic proxies with a WFR derivative requires an auxiliary backward pass. In our default protocol, this frozen score pass adds a \(1.21\times\) to \(1.27\times\) overhead to the total training time. We view this as a highly favorable trade-off. By paying a fixed computational cost during training, the controller secures higher reconstruction quality under strict primitive caps. Because the final model retains the same Gaussian cap and introduces no additional scoring operation at inference, the method preserves the inference-time architecture of the baseline.

\section{Full Fixed-Budget Comparisons}
\label{app:per_scene_results}

This section expands the fixed-budget comparison from the main text.  The wide table reports all dataset-level averages, including the 2M cap and random initialization rows.  The per-scene tables show where the aggregate gains come from and make clear which scenes and metrics are responsible for non-uniform behavior at larger budgets. 

\begin{table*}[!htbp]
\caption{Full fixed-budget comparison across datasets, capacities, and methods.
Each cell reports PSNR$\uparrow$ / SSIM$\uparrow$ / LPIPS$\downarrow$.
Rows marked $\dagger$ are published ConeGS benchmark values; Improved ADC and
Ours are evaluated at the same Gaussian caps.}
\label{tab:app_wide_sota}
\centering
\resizebox{\textwidth}{!}{
\begin{tabular}{@{}l l cccc@{}}
\toprule
\textbf{Dataset} & \textbf{Method} & \textbf{100k} & \textbf{500k} & \textbf{1M} & \textbf{2M} \\
& & (P / S / L) & (P / S / L) & (P / S / L) & (P / S / L) \\
\midrule
\multirow{7}{*}{\textbf{Mip-NeRF360}} & 3DGS~\cite{Kerbl2023GaussianSplatting} (SfM)$^\dagger$ & 23.61 / 0.693 / 0.413 & 28.22 / 0.821 / 0.260 & 28.71 / 0.846 / 0.222 & 29.01 / 0.861 / 0.200 \\
 & Improved ADC~\cite{Grubert2025ImprovingADC} & 26.34 / 0.781 / 0.322 & 28.46 / 0.843 / 0.232 & 28.94 / 0.858 / 0.205 & 29.21 / 0.871 / 0.184 \\
 & 3DGS-MCMC~\cite{Kheradmand2024MCMC} (SfM)$^\dagger$ & 27.06 / 0.800 / 0.303 & 28.82 / 0.861 / 0.214 & 29.23 / 0.875 / 0.190 & 29.52 / 0.883 / 0.172 \\
 & EDGS~\cite{Kotovenko2025EDGS}$^\dagger$ & 27.09 / 0.798 / 0.296 & 28.82 / 0.865 / 0.193 & 29.18 / 0.877 / 0.168 & 29.38 / 0.883 / 0.155 \\
 & ConeGS~\cite{Baranowski2025ConeGS}$^\dagger$ & 27.74 / 0.809 / 0.285 & 29.08 / 0.870 / 0.190 & 29.37 / 0.880 / 0.168 & 29.43 / 0.881 / 0.158 \\
 & Ours (SfM) & 27.94 / 0.818 / 0.276 & 29.11 / 0.869 / 0.203 & 29.49 / 0.883 / 0.178 & 29.66 / 0.890 / 0.163 \\
 & Ours (Random) & 26.71 / 0.769 / 0.337 & 28.79 / 0.856 / 0.222 & 29.19 / 0.874 / 0.190 & 29.42 / 0.885 / 0.170 \\
\midrule
\multirow{7}{*}{\textbf{Tanks \& Temples}} & 3DGS~\cite{Kerbl2023GaussianSplatting} (SfM)$^\dagger$ & 22.38 / 0.774 / 0.333 & 23.54 / 0.816 / 0.265 & 23.60 / 0.826 / 0.245 & 23.84 / 0.831 / 0.235 \\
 & Improved ADC~\cite{Grubert2025ImprovingADC} & 22.08 / 0.784 / 0.280 & 23.66 / 0.842 / 0.192 & 24.06 / 0.855 / 0.168 & 24.11 / 0.862 / 0.153 \\
 & 3DGS-MCMC~\cite{Kheradmand2024MCMC} (SfM)$^\dagger$ & 22.50 / 0.780 / 0.332 & 23.68 / 0.825 / 0.259 & 23.93 / 0.836 / 0.233 & 24.11 / 0.846 / 0.211 \\
 & EDGS~\cite{Kotovenko2025EDGS}$^\dagger$ & 22.32 / 0.777 / 0.324 & 23.47 / 0.831 / 0.229 & 23.66 / 0.842 / 0.200 & 23.47 / 0.846 / 0.183 \\
 & ConeGS~\cite{Baranowski2025ConeGS}$^\dagger$ & 23.12 / 0.791 / 0.310 & 23.69 / 0.829 / 0.229 & 23.70 / 0.835 / 0.207 & 23.69 / 0.837 / 0.192 \\
 & Ours (SfM) & 23.60 / 0.831 / 0.225 & 24.05 / 0.857 / 0.178 & 24.17 / 0.865 / 0.157 & 24.21 / 0.870 / 0.143 \\
 & Ours (Random) & 22.19 / 0.778 / 0.287 & 23.58 / 0.840 / 0.197 & 23.14 / 0.828 / 0.200 & 23.96 / 0.857 / 0.160 \\
\midrule
\multirow{7}{*}{\textbf{DeepBlending}} & 3DGS~\cite{Kerbl2023GaussianSplatting} (SfM)$^\dagger$ & 24.65 / 0.827 / 0.412 & 29.25 / 0.882 / 0.302 & 28.93 / 0.881 / 0.292 & 29.04 / 0.883 / 0.281 \\
 & Improved ADC~\cite{Grubert2025ImprovingADC} & 28.25 / 0.885 / 0.315 & 28.95 / 0.896 / 0.260 & 29.18 / 0.902 / 0.248 & 29.60 / 0.905 / 0.239 \\
 & 3DGS-MCMC~\cite{Kheradmand2024MCMC} (SfM)$^\dagger$ & 28.94 / 0.876 / 0.333 & 29.44 / 0.887 / 0.308 & 29.67 / 0.887 / 0.288 & 29.53 / 0.886 / 0.279 \\
 & EDGS~\cite{Kotovenko2025EDGS}$^\dagger$ & 28.43 / 0.872 / 0.337 & 29.33 / 0.889 / 0.286 & 29.43 / 0.889 / 0.271 & 29.38 / 0.887 / 0.258 \\
 & ConeGS~\cite{Baranowski2025ConeGS}$^\dagger$ & 29.44 / 0.880 / 0.328 & 29.86 / 0.891 / 0.285 & 29.69 / 0.889 / 0.273 & 29.76 / 0.886 / 0.261 \\
 & Ours (SfM) & 29.45 / 0.892 / 0.303 & 30.04 / 0.906 / 0.262 & 30.14 / 0.911 / 0.248 & 30.19 / 0.912 / 0.240 \\
 & Ours (Random) & 28.32 / 0.875 / 0.325 & 29.74 / 0.899 / 0.274 & 29.96 / 0.909 / 0.254 & 29.92 / 0.909 / 0.246 \\
\bottomrule
\end{tabular}
}
\end{table*}

The 100k per-scene results in Table~\ref{tab:perscene_100k} support the low-budget claim from the main text: when capacity is tight, the first-variation controller improves most scenes and gives especially strong LPIPS gains on Tanks \& Temples and DeepBlending. Tables~\ref{tab:perscene_500k} and~\ref{tab:perscene_1M} show the larger-cap regime.  The method remains strong on PSNR and SSIM, while LPIPS is more metric-dependent on Mip-NeRF360, matching the calibrated interpretation in the main text.

\begin{table*}[!htbp]
\caption{Per-scene fixed-budget comparison at 100k Gaussians.
Each cell reports PSNR / SSIM / LPIPS.
Baseline values for ConeGS, 3DGS (SfM), 3DGS-MCMC (SfM), and EDGS are
published ConeGS per-scene results; Ours is evaluated at the same cap.}
\label{tab:perscene_100k}
\centering
\scriptsize
\resizebox{\textwidth}{!}{%
\begin{tabular}{lccccc}
\toprule
Scene & ConeGS$^\dagger$ & 3DGS (SfM)$^\dagger$ & 3DGS-MCMC (SfM)$^\dagger$ & EDGS$^\dagger$ & \textbf{Ours} \\
\midrule
\multicolumn{6}{c}{\textbf{Mip-NeRF360}} \\
\midrule
Bicycle & 24.05 / 0.652 / 0.379 & 20.99 / 0.470 / 0.543 & 23.34 / 0.630 / 0.400 & 23.68 / 0.635 / 0.393 & \textbf{23.57 / 0.636 / 0.393} \\
Garden & 24.69 / 0.709 / 0.336 & 22.87 / 0.589 / 0.445 & 24.51 / 0.710 / 0.350 & 24.47 / 0.706 / 0.341 & \textbf{25.12 / 0.747 / 0.306} \\
Stump & 25.71 / 0.707 / 0.351 & 22.79 / 0.547 / 0.498 & 25.16 / 0.680 / 0.370 & 25.08 / 0.673 / 0.377 & \textbf{25.41 / 0.695 / 0.363} \\
Room & 31.09 / 0.903 / 0.254 & 27.02 / 0.859 / 0.328 & 30.37 / 0.900 / 0.270 & 30.08 / 0.896 / 0.265 & \textbf{30.94 / 0.903 / 0.261} \\
Counter & 28.22 / 0.879 / 0.251 & 25.64 / 0.845 / 0.307 & 27.83 / 0.880 / 0.260 & 27.77 / 0.877 / 0.251 & \textbf{28.44 / 0.895 / 0.232} \\
Kitchen & 29.73 / 0.896 / 0.183 & 20.53 / 0.671 / 0.440 & 28.37 / 0.890 / 0.210 & 28.84 / 0.891 / 0.186 & \textbf{30.63 / 0.917 / 0.152} \\
Bonsai & 30.67 / 0.920 / 0.241 & 25.41 / 0.866 / 0.331 & 29.84 / 0.910 / 0.260 & 29.73 / 0.906 / 0.256 & \textbf{31.48 / 0.934 / 0.225} \\
\midrule
\multicolumn{6}{c}{\textbf{Tanks \& Temples}} \\
\midrule
Truck & 24.54 / 0.822 / 0.291 & 23.59 / 0.803 / 0.318 & 23.86 / 0.810 / 0.316 & 23.33 / 0.799 / 0.314 & \textbf{25.12 / 0.861 / 0.200} \\
Train & 21.70 / 0.759 / 0.328 & 21.17 / 0.744 / 0.349 & 21.13 / 0.749 / 0.347 & 21.31 / 0.756 / 0.334 & \textbf{22.08 / 0.801 / 0.251} \\
\midrule
\multicolumn{6}{c}{\textbf{DeepBlending}} \\
\midrule
Dr Johnson & 28.78 / 0.875 / 0.338 & 23.58 / 0.808 / 0.440 & 28.16 / 0.869 / 0.343 & 27.66 / 0.865 / 0.347 & \textbf{29.22 / 0.887 / 0.298} \\
Playroom & 30.09 / 0.885 / 0.318 & 25.72 / 0.845 / 0.384 & 29.72 / 0.883 / 0.323 & 29.19 / 0.879 / 0.327 & \textbf{29.68 / 0.897 / 0.308} \\
\bottomrule
\end{tabular}}
\end{table*}

\begin{table*}[!htbp]
\caption{Per-scene fixed-budget comparison at 500k Gaussians.
Each cell reports PSNR / SSIM / LPIPS.
Baseline values for ConeGS, 3DGS (SfM), 3DGS-MCMC (SfM), and EDGS are
published ConeGS per-scene results; Ours is evaluated at the same cap.}
\label{tab:perscene_500k}
\centering
\scriptsize
\resizebox{\textwidth}{!}{%
\begin{tabular}{lccccc}
\toprule
Scene & ConeGS$^\dagger$ & 3DGS (SfM)$^\dagger$ & 3DGS-MCMC (SfM)$^\dagger$ & EDGS$^\dagger$ & \textbf{Ours} \\
\midrule
\multicolumn{6}{c}{\textbf{Mip-NeRF360}} \\
\midrule
Bicycle & 25.25 / 0.758 / 0.242 & 24.06 / 0.651 / 0.370 & 24.90 / 0.740 / 0.282 & 24.95 / 0.751 / 0.251 & \textbf{25.14 / 0.753 / 0.265} \\
Garden & 26.87 / 0.837 / 0.157 & 25.16 / 0.743 / 0.296 & 26.36 / 0.823 / 0.189 & 26.49 / 0.833 / 0.161 & \textbf{26.79 / 0.837 / 0.170} \\
Stump & 27.03 / 0.793 / 0.215 & 25.27 / 0.688 / 0.348 & 26.81 / 0.779 / 0.250 & 26.50 / 0.768 / 0.237 & \textbf{27.21 / 0.795 / 0.229} \\
Room & 32.02 / 0.924 / 0.205 & 31.47 / 0.913 / 0.234 & 31.72 / 0.921 / 0.221 & 31.36 / 0.923 / 0.203 & \textbf{31.84 / 0.922 / 0.219} \\
Counter & 28.87 / 0.906 / 0.194 & 28.86 / 0.901 / 0.213 & 28.95 / 0.907 / 0.205 & 29.05 / 0.912 / 0.184 & \textbf{29.06 / 0.911 / 0.200} \\
Kitchen & 31.38 / 0.929 / 0.124 & 30.76 / 0.918 / 0.143 & 31.04 / 0.921 / 0.141 & 31.35 / 0.928 / 0.122 & \textbf{31.57 / 0.926 / 0.136} \\
Bonsai & 32.16 / 0.944 / 0.191 & 31.95 / 0.936 / 0.218 & 31.97 / 0.939 / 0.210 & 32.02 / 0.942 / 0.191 & \textbf{32.17 / 0.942 / 0.205} \\
\midrule
\multicolumn{6}{c}{\textbf{Tanks \& Temples}} \\
\midrule
Truck & 25.41 / 0.858 / 0.204 & 24.80 / 0.842 / 0.253 & 25.24 / 0.851 / 0.242 & 24.66 / 0.850 / 0.212 & \textbf{26.03 / 0.888 / 0.145} \\
Train & 21.97 / 0.799 / 0.253 & 22.28 / 0.790 / 0.277 & 22.12 / 0.798 / 0.275 & 22.27 / 0.812 / 0.246 & \textbf{22.07 / 0.826 / 0.211} \\
\midrule
\multicolumn{6}{c}{\textbf{DeepBlending}} \\
\midrule
Dr Johnson & 29.20 / 0.890 / 0.291 & 28.85 / 0.881 / 0.307 & 28.71 / 0.884 / 0.313 & 28.60 / 0.888 / 0.292 & \textbf{29.70 / 0.902 / 0.255} \\
Playroom & 30.51 / 0.892 / 0.279 & 29.66 / 0.883 / 0.298 & 30.16 / 0.889 / 0.303 & 30.06 / 0.890 / 0.281 & \textbf{30.38 / 0.911 / 0.270} \\
\bottomrule
\end{tabular}}
\end{table*}

\begin{table*}[!htbp]
\caption{Per-scene fixed-budget comparison at 1M Gaussians.
Each cell reports PSNR / SSIM / LPIPS.
Baseline values for ConeGS, 3DGS (SfM), 3DGS-MCMC (SfM), and EDGS are
published ConeGS per-scene results; Ours is evaluated at the same cap.}
\label{tab:perscene_1M}
\centering
\scriptsize
\resizebox{\textwidth}{!}{%
\begin{tabular}{lccccc}
\toprule
Scene & ConeGS$^\dagger$ & 3DGS (SfM)$^\dagger$ & 3DGS-MCMC (SfM)$^\dagger$ & EDGS$^\dagger$ & \textbf{Ours} \\
\midrule
\multicolumn{6}{c}{\textbf{Mip-NeRF360}} \\
\midrule
Bicycle & 25.45 / 0.778 / 0.198 & 24.40 / 0.688 / 0.325 & 25.34 / 0.768 / 0.238 & 25.26 / 0.778 / 0.199 & \textbf{25.51 / 0.780 / 0.220} \\
Garden & 27.42 / 0.861 / 0.115 & 26.70 / 0.827 / 0.172 & 26.83 / 0.847 / 0.147 & 27.05 / 0.857 / 0.118 & \textbf{27.45 / 0.862 / 0.126} \\
Stump & 27.23 / 0.802 / 0.184 & 25.64 / 0.719 / 0.301 & 27.21 / 0.798 / 0.213 & 26.71 / 0.783 / 0.203 & \textbf{27.56 / 0.814 / 0.193} \\
Room & 32.30 / 0.928 / 0.196 & 31.62 / 0.918 / 0.221 & 32.09 / 0.926 / 0.208 & 31.71 / 0.928 / 0.191 & \textbf{32.10 / 0.927 / 0.205} \\
Counter & 29.21 / 0.911 / 0.181 & 29.11 / 0.907 / 0.201 & 29.24 / 0.913 / 0.191 & 29.27 / 0.917 / 0.171 & \textbf{29.34 / 0.917 / 0.187} \\
Kitchen & 31.66 / 0.932 / 0.117 & 31.18 / 0.924 / 0.131 & 31.44 / 0.927 / 0.130 & 31.83 / 0.933 / 0.114 & \textbf{32.00 / 0.932 / 0.124} \\
Bonsai & 32.33 / 0.945 / 0.183 & 32.34 / 0.941 / 0.205 & 32.46 / 0.944 / 0.199 & 32.44 / 0.947 / 0.181 & \textbf{32.49 / 0.947 / 0.194} \\
\midrule
\multicolumn{6}{c}{\textbf{Tanks \& Temples}} \\
\midrule
Truck & 25.34 / 0.864 / 0.181 & 25.06 / 0.851 / 0.234 & 25.57 / 0.862 / 0.217 & 25.02 / 0.862 / 0.181 & \textbf{26.21 / 0.897 / 0.120} \\
Train & 22.05 / 0.805 / 0.232 & 22.13 / 0.801 / 0.255 & 22.29 / 0.811 / 0.249 & 22.29 / 0.823 / 0.220 & \textbf{22.12 / 0.834 / 0.194} \\
\midrule
\multicolumn{6}{c}{\textbf{DeepBlending}} \\
\midrule
Dr Johnson & 28.98 / 0.886 / 0.282 & 28.85 / 0.884 / 0.293 & 29.06 / 0.884 / 0.291 & 28.75 / 0.890 / 0.277 & \textbf{29.73 / 0.907 / 0.242} \\
Playroom & 30.39 / 0.891 / 0.264 & 29.02 / 0.877 / 0.290 & 30.28 / 0.890 / 0.284 & 30.11 / 0.889 / 0.265 & \textbf{30.55 / 0.915 / 0.255} \\
\bottomrule
\end{tabular}}
\end{table*}

\FloatBarrier
\section{Budget Curves and Runtime}
\label{app:budget_runtime}

This section supports the main-paper claim that the effect of the controller is
largest when the Gaussian budget is tight and reports the training overhead
introduced by the score pass.  Figure~\ref{fig:app_budget_frontier} extends
the main PSNR--budget frontier to PSNR, SSIM, and LPIPS under both SfM and
random initialization.

\begin{figure*}[!htbp]
\centering
\includegraphics[width=1.0\textwidth]{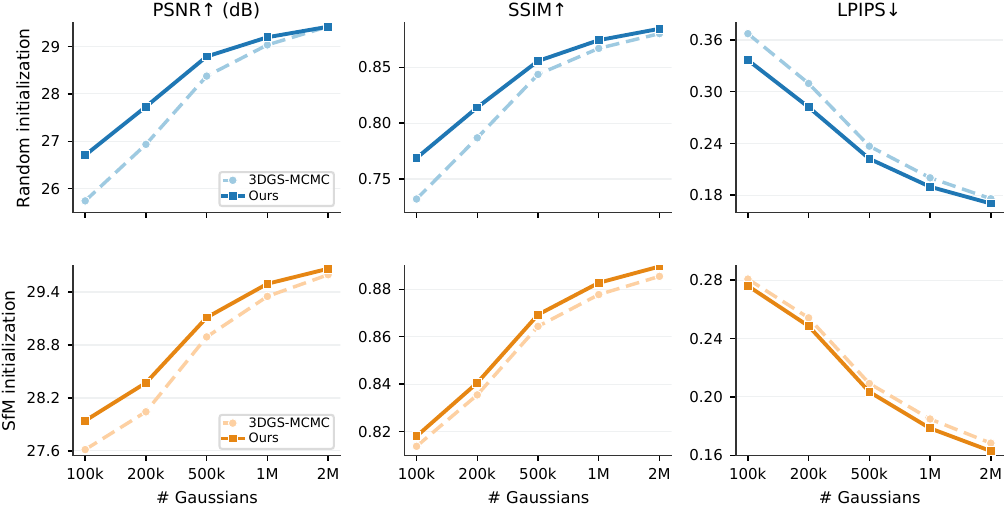}
\caption{
Quality as a function of Gaussian budget on the Mip-NeRF360 comparison
set under random and SfM initialization. Each point is a seed-0 dataset average
at a fixed Gaussian cap; error bands are not shown because matched multi-seed
runs are not available for this sweep. Panels use tight y-axis ranges within
each initialization and metric to make the matched Ours-vs-3DGS-MCMC gap
readable. Across PSNR, SSIM, and LPIPS, the proposed controller improves over
the matched 3DGS-MCMC backend, with the largest gains at lower budgets.
}
\label{fig:app_budget_frontier}
\end{figure*}

Figure~\ref{fig:app_quality_time} and Table~\ref{tab:app_pareto_values} show
the cost of the additional score pass.  Training time increases by roughly
$1.21$--$1.27\times$ in this protocol, while inference cost remains tied to the
same fixed Gaussian budget.

\begin{figure}[!htbp]
\centering
\includegraphics[width=0.75\linewidth]{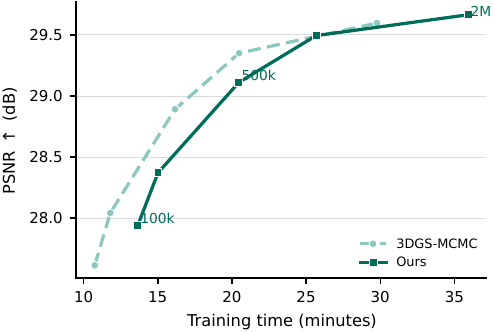}
\caption{
Quality--time tradeoff on the Mip-NeRF360 SfM setting.
Ours improves PSNR at matched Gaussian caps while adding score-pass overhead
during training; inference cost remains tied to the same fixed Gaussian budget.
}
\label{fig:app_quality_time}
\end{figure}

\begin{table}[!htbp]
\caption{Quality--time and computational efficiency values on the Mip-NeRF360 SfM setting. MCMC time is estimated by subtracting measured
score-pass time from the matched Ours wall-clock logs. Peak memory and FPS are
omitted because complete matched measurements were not logged for this timing
protocol.}
\label{tab:efficiency_single_col}
\label{tab:app_pareto_values}
\centering
\resizebox{\linewidth}{!}{
\begin{tabular}{@{}c l c c c c c c c@{}}
\toprule
\textbf{Cap} & \textbf{Method} & \textbf{PSNR$\uparrow$} & \textbf{SSIM$\uparrow$} & \textbf{LPIPS$\downarrow$} & \textbf{Train time} & \textbf{Overhead} & \textbf{Score-pass} & \textbf{Base train} \\
\midrule
\multirow{2}{*}{100k} & 3DGS-MCMC & 27.61 & 0.814 & 0.281 & 10.75m & 1.00$\times$ & \na & \na \\
                      & Ours      & \best{27.94} & \best{0.818} & \best{0.276} & 13.65m & 1.27$\times$ & 2.90m & 10.75m \\
\addlinespace
\multirow{2}{*}{200k} & 3DGS-MCMC & 28.04 & 0.835 & 0.254 & 11.79m & 1.00$\times$ & \na & \na \\
                      & Ours      & \best{28.37} & \best{0.841} & \best{0.249} & 15.01m & 1.27$\times$ & 3.22m & 11.79m \\
\addlinespace
\multirow{2}{*}{500k} & 3DGS-MCMC & 28.89 & 0.864 & 0.209 & 16.16m & 1.00$\times$ & \na & \na \\
                      & Ours      & \best{29.11} & \best{0.869} & \best{0.203} & 20.47m & 1.27$\times$ & 4.31m & 16.16m \\
\addlinespace
\multirow{2}{*}{1M} & 3DGS-MCMC & 29.35 & 0.878 & 0.185 & 20.48m & 1.00$\times$ & \na & \na \\
                      & Ours      & \best{29.49} & \best{0.883} & \best{0.178} & 25.68m & 1.25$\times$ & 5.19m & 20.48m \\
\addlinespace
\multirow{2}{*}{2M} & 3DGS-MCMC & 29.60 & 0.885 & 0.168 & 29.78m & 1.00$\times$ & \na & \na \\
                      & Ours      & \best{29.66} & \best{0.890} & \best{0.163} & 35.92m & 1.21$\times$ & 6.14m & 29.78m \\
\bottomrule
\end{tabular}
}
\end{table}

\FloatBarrier
\section{Score Validation and Controller Sensitivity}
\label{app:score_validation}

The main ablation proves that the WFR statistic improves final reconstruction quality. This section isolates the score itself to verify that it acts as a true local descent surrogate. To test the score's predictive power, we design three empirical probes: \textbf{Copy gain} measures the immediate held-out loss reduction achieved by duplicating a high-score Gaussian, \textbf{Recycle hit rate} tracks how often removing a low-score Gaussian safely leaves the loss unaffected or improved, \textbf{Removal agreement} computes the rank correlation between our score and the actual finite damage caused by deleting a primitive.

Figure~\ref{fig:app_utility_control_validation} shows the results. High-score Gaussians are demonstrably better birth sites. Low-score Gaussians are significantly safer to recycle. Crucially, the WFR score strongly agrees with actual finite removal damage, whereas opacity and position gradients show weak or near-random correlation.

\begin{figure*}[!htbp]
\centering
\includegraphics[width=0.98\textwidth]{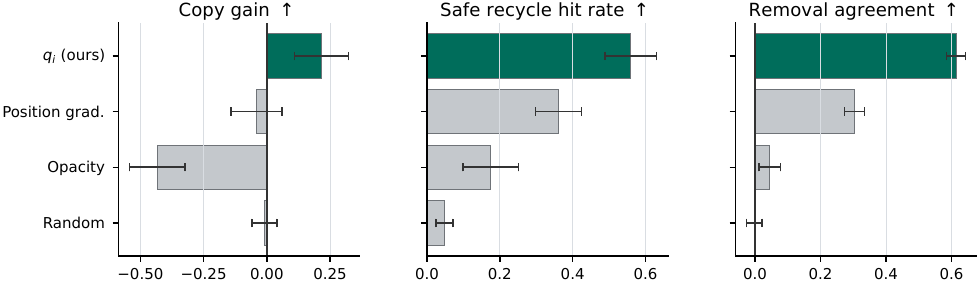}

\vspace{0.6em}

\small
\setlength{\tabcolsep}{6pt}
\begin{tabular}{lccc}
\toprule
\textbf{Method}
& \textbf{Copy gain$\uparrow$}
& \textbf{Recycle hit rate$\uparrow$}
& \textbf{Removal agree.$\uparrow$} \\
\midrule
\textbf{Ours score} ($q_i$)
& $\mathbf{0.216{\pm}0.107}$
& $\mathbf{0.560{\pm}0.070}$
& $\mathbf{0.613{\pm}0.029}$ \\
Position grad.
& $-0.041{\pm}0.101$
& $0.361{\pm}0.063$
& $0.303{\pm}0.031$ \\
Opacity
& $-0.434{\pm}0.110$
& $0.175{\pm}0.076$
& $0.044{\pm}0.033$ \\
Random
& $-0.010{\pm}0.050$
& $0.048{\pm}0.023$
& $-0.003{\pm}0.024$ \\
\bottomrule
\end{tabular}

\caption{
Utility-score validation for the fixed-budget controller.
Top: choosing high-score Gaussians gives larger held-out loss reduction after a
small test copy; choosing low-score Gaussians better identifies slots that are
safe to recycle; and the score agrees with finite removal damage.
Bottom: mean $\pm$ standard error over disjoint-view probes. Copy and recycle
use 18 samples; removal agreement uses 16 finite Spearman correlations.
}
\label{fig:app_utility_control_validation}
\end{figure*}

Table~\ref{tab:app_probe_view_split} separates same-view and disjoint-view validation. Same-view probes are naturally easier because the derivative and probe loss share the same views. However, performance on disjoint and held-out test views remains positive, confirming that the WFR score identifies true geometric utility rather than merely memorizing the local score batch.

\begin{table}[!htbp]
\caption{
Effect of probe-view choice in utility-score validation.
The main paper uses disjoint probe views to avoid measuring the derivative on
the same views used to compute the score. Values report mean $\pm$ standard
error over 18 validation samples.
}
\label{tab:app_probe_view_split}
\centering
\small
\setlength{\tabcolsep}{6pt}
\begin{tabular}{l c c c}
\toprule
\textbf{Probe setting} & \textbf{Spearman$\uparrow$} & \textbf{Top-10\% gain$\uparrow$} & \textbf{Death P@10$\uparrow$} \\
\midrule
Same views as score batch & $0.954 \pm 0.007$ & $1.442 \pm 0.081$ & $0.929 \pm 0.026$ \\
Disjoint probe views      & $0.613 \pm 0.029$ & $0.216 \pm 0.107$ & $0.560 \pm 0.070$ \\
Held-out test views       & $0.632 \pm 0.026$ & $0.600 \pm 0.108$ & $0.615 \pm 0.055$ \\
\bottomrule
\end{tabular}
\end{table}

\FloatBarrier
\section{Additional Qualitative Results}
\label{app:qualitative}

The qualitative results illustrate the visual consequences of the matched controller comparison.  They are not used as the primary evidence for the first-variation claim; the quantitative tables and controller ablations serve that role.  The figures below show that the same fixed-cap controller improves fine structures, boundaries, and repeated textures across datasets and initialization regimes.

\begin{figure*}[!htbp]
\centering
\begin{overpic}[width=1.0\textwidth]{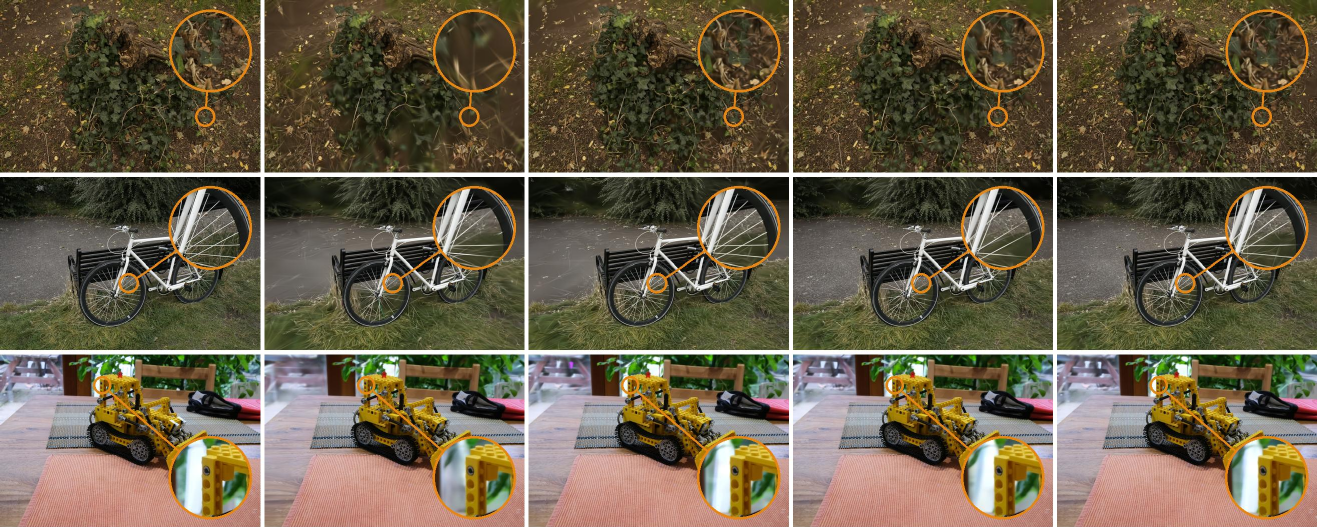}
    \put(8, -2.5){\small \textbf{GT}}
    \put(23, -2.5){\small \textbf{MCMC-Rand.}}
    \put(44, -2.5){\small \textbf{Ours-Rand.}}
    \put(63, -2.5){\small \textbf{MCMC-SfM}}
    \put(85, -2.5){\small \textbf{Ours-SfM}}
\end{overpic}
\par\vspace{0.7em} %
\caption{
Qualitative comparison under random and SfM initialization. For each scene, we
compare the matched 3DGS-MCMC controller and our score-guided controller under
both initialization regimes. This comparison verifies that the visual gains are
not tied to a particular initialization choice. All columns use the same
Gaussian cap, test view, and crop coordinates within each row.
}
\label{fig:app_initialization_comparison}
\end{figure*}

\begin{figure*}[!htbp]
\centering
\begin{overpic}[width=1.0\textwidth]{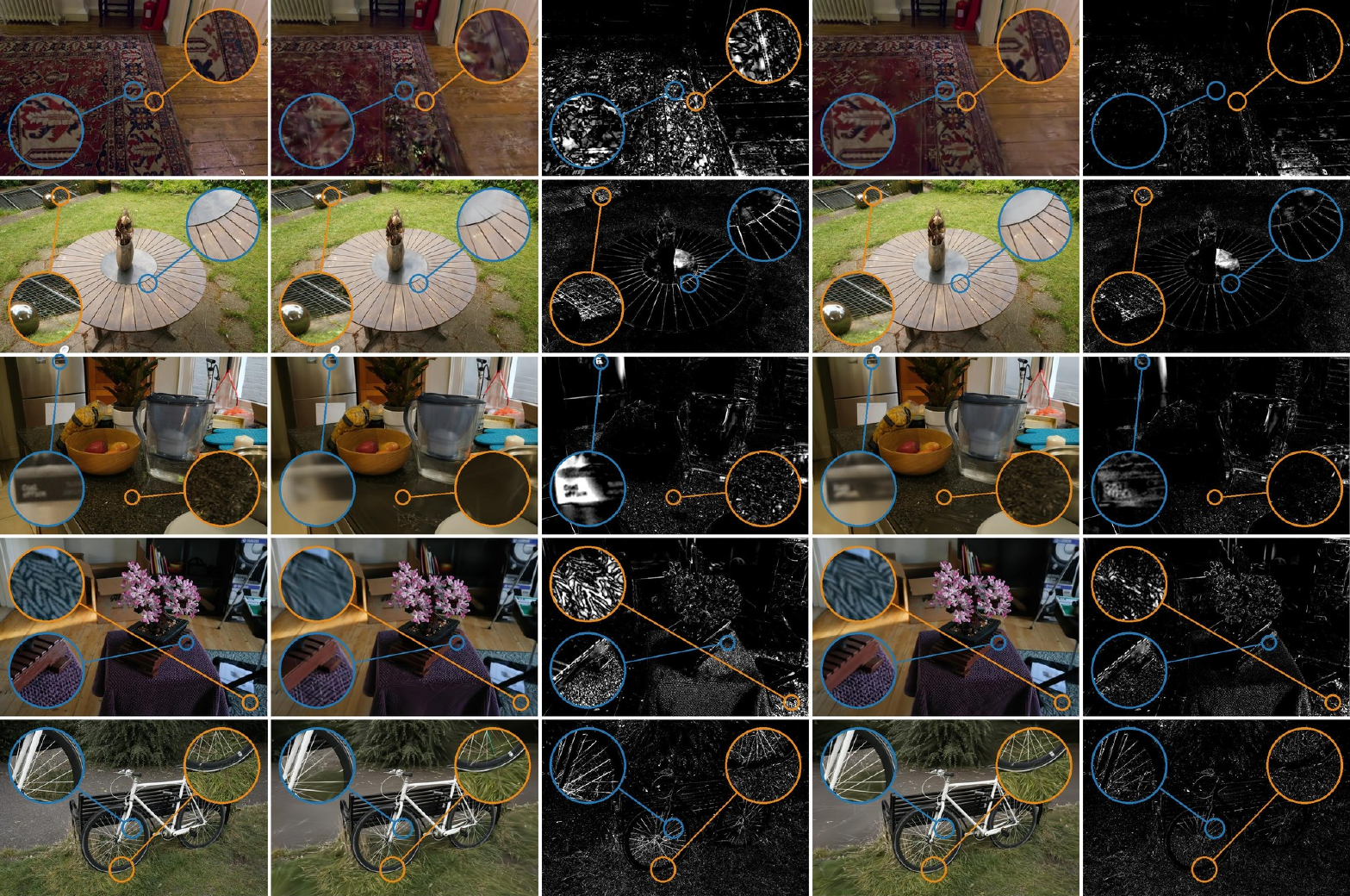}
    \put(3, -2.8){\small \textbf{Ground Truth}}
    \put(23, -2.8){\small \textbf{3DGS-MCMC}}
    \put(45, -2.8){\small \textbf{$L_1$ Error}}
    \put(67, -2.8){\small \textbf{Ours}}
    \put(86, -2.8){\small \textbf{$L_1$ Error}}
\end{overpic}
\par\vspace{0.7em} %
\caption{
Qualitative comparison across scenes and budgets.
Each row shows one test-view crop with a fixed scene, Gaussian cap, and random
initialization. From top to bottom, the rows correspond to
\emph{Dr.\ Johnson} (1M, random), \emph{Garden} (500k, random),
\emph{Counter} (200k, random), \emph{Bonsai} (100k, random), and
\emph{Bicycle} (100k, random). From left to right, we show the ground-truth
image, the matched 3DGS-MCMC rendering, the per-pixel absolute $L_1$ error map
$|I_{\mathrm{GT}} - I_{\mathrm{MCMC}}|$, our rendering, and the per-pixel
absolute $L_1$ error map $|I_{\mathrm{GT}} - I_{\mathrm{ours}}|$. Brighter
pixels in the error maps indicate larger reconstruction error. All methods use
the same test view, scale, and crop coordinates within each row.
}
\label{fig:app_qualitative_scenes_budgets}
\end{figure*}

\begin{figure*}[!htbp]
\centering
\begin{overpic}[width=1.0\textwidth]{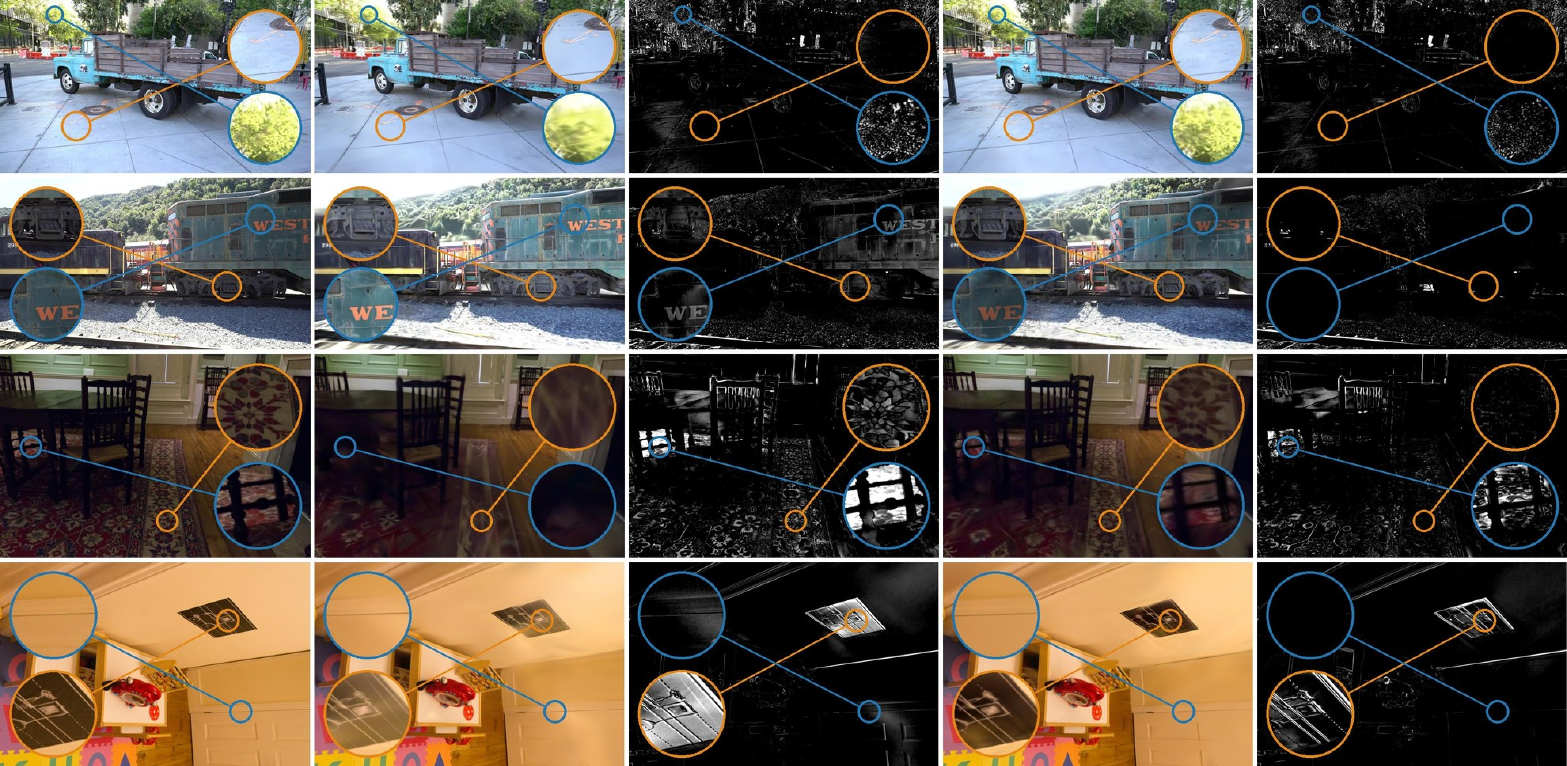}
    \put(3, -2.8){\small \textbf{Ground Truth}}
    \put(22, -2.8){\small \textbf{3DGS-MCMC}}
    \put(45, -2.8){\small \textbf{$L_1$ Error}}
    \put(67, -2.8){\small \textbf{Ours}}
    \put(84, -2.8){\small \textbf{$L_1$ Error}}
\end{overpic}
\par\vspace{0.7em} %
\caption{
All-scene qualitative comparison on Tanks \& Temples and DeepBlending at a
200k-Gaussian cap with random initialization. Each row uses the same test view
and crop coordinates for the ground truth, 3DGS-MCMC, and our method. These
examples show that the score-guided controller also improves visual fidelity
outside Mip-NeRF360, especially on thin structures, object boundaries, and
repeated textures.
}
\label{fig:app_all_scene_tandt_deepblending}
\end{figure*}

\begin{figure*}[!htbp]
\centering
\begin{overpic}[width=1.0\textwidth]{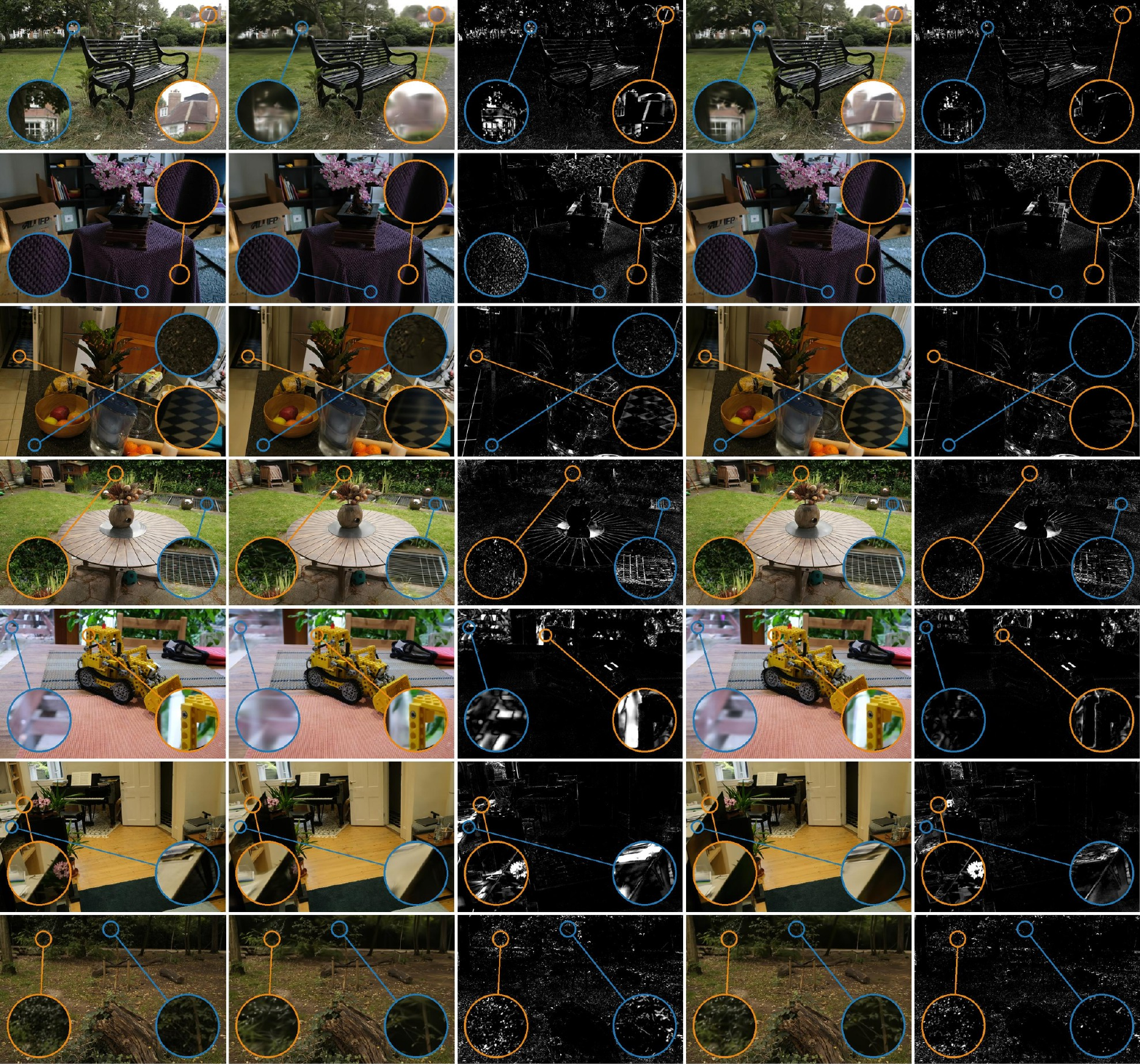}
    \put(3, -2.8){\small \textbf{Ground Truth}}
    \put(22, -2.8){\small \textbf{3DGS-MCMC}}
    \put(45, -2.8){\small \textbf{$L_1$ Error}}
    \put(67, -2.8){\small \textbf{Ours}}
    \put(84, -2.8){\small \textbf{$L_1$ Error}}
\end{overpic}
\par\vspace{0.7em} %
\caption{
All-scene qualitative comparison on Mip-NeRF360 at a 200k-Gaussian cap with
random initialization. Each row corresponds to one Mip-NeRF360 scene, and all
methods use identical test views and crop coordinates. This figure complements
the representative examples in Figure~\ref{fig:main_qualitative} by showing
that the visual improvement is not restricted to a small number of selected
scenes.
}
\label{fig:app_all_scene_mipnerf360}
\end{figure*}
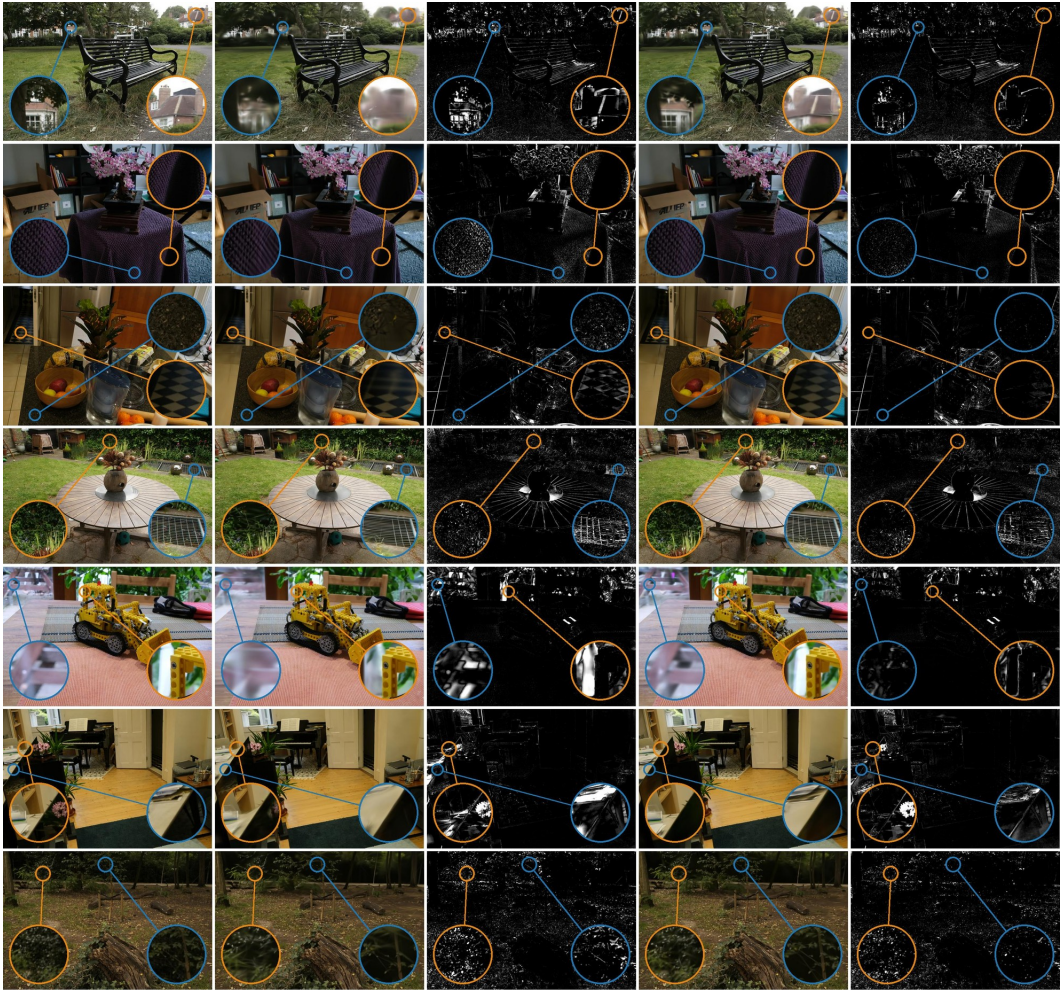

\clearpage
\section*{NeurIPS Paper Checklist}
 
\begin{enumerate}

\item {\bf Claims}
    \item[] Question: Do the main claims made in the abstract and introduction accurately reflect the paper's contributions and scope?
    \item[] Answer: \answerYes{} %
    \item[] Justification: The abstract and Section~\ref{sec:introduction} state the proposed first-variation birth--death controller, its fixed-budget 3DGS scope, and the main empirical claim. The method is formalized in Section~\ref{sec:method}, especially Eqs.~\eqref{eq:method_first_variation_score}--\eqref{eq:mcmc_realization}, and the claims are supported by Table~\ref{tab:sota_comparison}, Table~\ref{tab:controller_ablation}, Figure~\ref{fig:main_psnr_budget_frontier}, and Appendices~\ref{app:dropin_mcmc}, \ref{app:per_scene_results}, and \ref{app:score_validation}.

\item {\bf Limitations}
    \item[] Question: Does the paper discuss the limitations of the work performed by the authors?
    \item[] Answer: \answerYes{} %
    \item[] Justification: The paper states that the score is a local mass-coordinate descent surrogate rather than a finite-step global guarantee in Section~\ref{subsec:birth_death_score_gap} and Section~\ref{subsec:mcmc_realization}. Additional limitations and diagnostics are discussed in Appendix~\ref{app:mass_opacity_gradient}, Appendix~\ref{app:train_failure_case}, and Appendix~\ref{app:computation_tradeoff}.

\item {\bf Theory assumptions and proofs}
    \item[] Question: For each theoretical result, does the paper provide the full set of assumptions and a complete (and correct) proof?
    \item[] Answer: \answerYes{} %
    \item[] Justification: Proposition~\ref{prop:qi_first_variation}, Proposition~\ref{prop:descent_lemma}, and Proposition~\ref{prop:low_opacity_bound} state their assumptions in Section~\ref{sec:method}. Complete proofs are provided in Appendix~\ref{app:deferred_proofs}.

    \item {\bf Experimental result reproducibility}
    \item[] Question: Does the paper fully disclose all the information needed to reproduce the main experimental results of the paper to the extent that it affects the main claims and/or conclusions of the paper (regardless of whether the code and data are provided or not)?
    \item[] Answer: \answerYes{} %
    \item[] Justification: Section~\ref{subsec:experimental_settings} specifies datasets, metrics, budgets, baselines, and the matched-comparison protocol. Appendix~\ref{app:reproducibility} gives the code-faithful algorithm, hyperparameters, optimizer/training details, hardware/software environment, and the anonymized supplementary code reproduces the proposed method.

\item {\bf Open access to data and code}
    \item[] Question: Does the paper provide open access to the data and code, with sufficient instructions to faithfully reproduce the main experimental results, as described in supplemental material?
    \item[] Answer: \answerYes{} %
    \item[] Justification: An anonymized supplementary code package is submitted with training and evaluation scripts, configuration files, dependency instructions, and commands for reproducing the proposed method. The standard benchmark datasets are not redistributed; the supplementary README gives public data-access and preprocessing instructions, and rows marked $\dagger$ use cited published benchmark numbers as described in Section~\ref{subsec:experimental_settings}. 

\item {\bf Experimental setting/details}
    \item[] Question: Does the paper specify all the training and test details (e.g., data splits, hyperparameters, how they were chosen, type of optimizer) necessary to understand the results?
    \item[] Answer: \answerYes{} %
    \item[] Justification: Section~\ref{subsec:experimental_settings} specifies the benchmark datasets, held-out metrics, Gaussian caps, baseline protocol, and training length. Appendix~\ref{app:reproducibility}, Table~\ref{tab:app_config}, and Algorithm~\ref{alg:first_variation_birth_death} provide the controller interval, score batch size, EMA/clipping settings, opacity threshold, regularization weights, optimizer setting, GPU, CUDA/PyTorch versions, and evaluation metrics.

\item {\bf Experiment statistical significance}
    \item[] Question: Does the paper report error bars suitably and correctly defined or other appropriate information about the statistical significance of the experiments?
    \item[] Answer: \answerNo{} %
    \item[] Justification: The main fixed-budget reconstruction sweeps report dataset-level and per-scene averages rather than multi-seed error bars, because full matched multi-seed 3DGS sweeps are computationally expensive. The paper nevertheless provides controlled matched reruns, per-scene tables in Appendix~\ref{app:per_scene_results}, and mean $\pm$ standard-error score-validation probes in Figure~\ref{fig:app_utility_control_validation} and Table~\ref{tab:app_probe_view_split}.

\item {\bf Experiments compute resources}
    \item[] Question: For each experiment, does the paper provide sufficient information on the computer resources (type of compute workers, memory, time of execution) needed to reproduce the experiments?
    \item[] Answer: \answerNo{} %
    \item[] Justification: Appendix~\ref{app:reproducibility} reports the GPU and software environment, and Appendix~\ref{app:budget_runtime} reports wall-clock training times and score-pass overhead in Figure~\ref{fig:app_quality_time} and Table~\ref{tab:app_pareto_values}. Complete peak-memory measurements and total project-level compute are not available for all experiments, and Table~\ref{tab:app_pareto_values} explicitly discloses omitted peak-memory/FPS logs.
    
\item {\bf Code of ethics}
    \item[] Question: Does the research conducted in the paper conform, in every respect, with the NeurIPS Code of Ethics \url{https://neurips.cc/public/EthicsGuidelines}?
    \item[] Answer: \answerYes{} %
    \item[] Justification: The research conforms to the NeurIPS Code of Ethics: it uses public benchmark datasets, does not collect new human-subject data, does not involve crowdsourcing, and submits anonymized supplementary code for review. The work is a reconstruction-method paper and does not release private data or high-risk pretrained models.

\item {\bf Broader impacts}
    \item[] Question: Does the paper discuss both potential positive societal impacts and negative societal impacts of the work performed?
    \item[] Answer: \answerYes{} %
    \item[] Justification: Section~\ref{sec:introduction} notes positive applications such as immersive rendering, robotics, digital humans, and geometric reconstruction. Potential negative uses include more accurate reconstruction of privacy-sensitive scenes or surveillance settings; this work mitigates direct risk by using only public benchmarks and releasing research code rather than private scene data or sensitive pretrained assets.
    
\item {\bf Safeguards}
    \item[] Question: Does the paper describe safeguards that have been put in place for responsible release of data or models that have a high risk for misuse (e.g., pre-trained language models, image generators, or scraped datasets)?
    \item[] Answer: \answerNA{} %
    \item[] Justification: The paper does not release a high-risk pretrained model, image generator, language model, scraped dataset, or private dataset. The released asset is anonymized research code for training and evaluating a 3D reconstruction method on public benchmarks.

\item {\bf Licenses for existing assets}
    \item[] Question: Are the creators or original owners of assets (e.g., code, data, models), used in the paper, properly credited and are the license and terms of use explicitly mentioned and properly respected?
    \item[] Answer: \answerYes{} %
    \item[] Justification: The paper cites the original datasets, metrics, baselines, and codebases used in Section~\ref{subsec:experimental_settings} and the references. The supplementary README lists the corresponding data/code sources, versions or commits when available, licenses or terms of use, and states that benchmark datasets are not redistributed.

\item {\bf New assets}
    \item[] Question: Are new assets introduced in the paper well documented and is the documentation provided alongside the assets?
    \item[] Answer: \answerYes{} %
    \item[] Justification: The new asset is the anonymized supplementary source code. The code package includes documentation, environment setup, training/evaluation commands, configuration files, and license/terms information.

\item {\bf Crowdsourcing and research with human subjects}
    \item[] Question: For crowdsourcing experiments and research with human subjects, does the paper include the full text of instructions given to participants and screenshots, if applicable, as well as details about compensation (if any)? 
    \item[] Answer: \answerNA{} %
    \item[] Justification: The paper does not involve crowdsourcing, user studies, human-subject experiments, or newly collected human-subject data. All experiments use existing public benchmark datasets.

\item {\bf Institutional review board (IRB) approvals or equivalent for research with human subjects}
    \item[] Question: Does the paper describe potential risks incurred by study participants, whether such risks were disclosed to the subjects, and whether Institutional Review Board (IRB) approvals (or an equivalent approval/review based on the requirements of your country or institution) were obtained?
    \item[] Answer: \answerNA{} %
    \item[] Justification: The paper does not involve crowdsourcing or research with human subjects, so IRB approval or equivalent review is not applicable.

\item {\bf Declaration of LLM usage}
    \item[] Question: Does the paper describe the usage of LLMs if it is an important, original, or non-standard component of the core methods in this research? Note that if the LLM is used only for writing, editing, or formatting purposes and does \emph{not} impact the core methodology, scientific rigor, or originality of the research, declaration is not required.
    \item[] Answer: \answerNA{} %
    \item[] Justification: LLMs are not an important, original, or non-standard component of the core method, experiments, code, or analysis. Any use of LLMs only for writing, editing, or formatting does not affect the methodology, scientific rigor, or originality of the research.

\end{enumerate}

\end{document}